%% file: main.tex
\documentclass[11pt]{article}
\usepackage{style}
\usepackage[left=1in,bottom=1in,top=1in,right=1in]{geometry}
\hypersetup{%
	colorlinks   = true, %
	urlcolor     = blue, %
	linkcolor    = blue, %
	citecolor    = blue,  %
}

\title{The Optimal Sample Complexity of Multiclass and List Learning}
\author{Chirag Pabbaraju\thanks{Stanford University. Email: \texttt{cpabbara@cs.stanford.edu.}}}
\date{\today}

\begin{document}

\maketitle

\begin{abstract}
    While the optimal sample complexity of binary classification in terms of the VC dimension is well-established, determining the optimal sample complexity of multiclass classification had remained open. The appropriate complexity parameter for multiclass classification is the DS dimension, and despite significant efforts, a gap of $\sqrt{\text{DS}}$ had persisted between the upper and lower bounds on sample complexity.
    
    Recent work by Hanneke et al. (2026) shows a novel algebraic characterization of multiclass hypothesis classes in terms of their DS dimension. Building up on this, we show that the maximum hypergraph density of any multiclass hypothesis class is upper-bounded by its DS dimension. This proves a longstanding conjecture of Daniely and Shalev-Shwartz (2014). As a consequence, we determine the optimal dependence of the sample complexity on the DS dimension for multiclass as well as list learning.
\end{abstract}

\newpage
\input{introduction}

\input{proof-sketches}
\input{preliminaries}
\input{main-upper-bound}
\input{list-learning-results}
\input{discussion}

\section*{Acknowledgements}
This work is supported by a Google PhD Fellowship, and Moses Charikar's and Gregory Valiant's Simons Investigator Awards. The author would like to thank Oliver Janzer and Ishaq Aden-Ali for helpful discussions about this problem. Finally, the assistance of ChatGPT as an all-round sounding board is gratefully acknowledged. In particular, it significantly helped with various technical details involved in the agnostic list learning results.

\bibliographystyle{alpha}
\bibliography{references}

\appendix
\input{appendix}

\end{document}

%% file: introduction.tex
\section{Introduction}
\label{sec:intro}

Classification is a foundational task in  machine learning, and understanding the minimum number of training samples required to achieve a desired classification accuracy, or determining the \textit{sample complexity} of the classification task, is a core pursuit in learning theory. In the classification task, we seek to assign labels from a label space $\mcY$ to unlabeled points in a data domain $\mcX$. The Probably-Approximately-Correct or PAC model of Valiant \cite{valiant1984theory} defines a concrete theoretical framework for this task. In binary classification, where the number of classes $|\mcY|=2$, the optimal sample complexity of the learning task is completely understood. Namely, if the data is labeled, or rather \textit{realizable}, by some target hypothesis from a hypothesis class $\mcH$, then $\Theta(\dvc)$ samples, where $\dvc$ is a measure of complexity of the class known as the \textit{VC dimension} \cite{vapnik1974theory}, are both sufficient and necessary. In slightly more detail, one provably requires
\begin{align}
    \label{eqn:binary-optimal-sample-complexity}
    \Theta\left(\frac{\dvc+\log(1/\delta)}{\eps}\right)
\end{align}
many samples from a generic data distribution $\mcD$, in order to output, with probability $1-\delta$, a hypothesis $\widehat{f}$ that classifies with error $\err_{\mcD}(\widehat{f}) := \Pr_{(x,y) \sim \mcD}[\widehat{f}(x) \neq y]$ at most $\eps$ on a new test point \cite{hanneke2016optimal}. A similar exact characterization for binary classification also exists in the more general \textit{agnostic} setting, where the data is not necessarily labeled by any hypothesis in the class. Here, the dependence on the error parameter $\eps$ degrades from $1/\eps$ to $1/\eps^2$.

Despite such a clean picture for binary classification, the landscape of multiclass classification, where the number of classes $|\mcY| > 2$, had still not been %
as exactly figured out. Keeping the sample complexity aspect aside, even the problem of determining a complexity parameter similar to the VC dimension which provably characterizes the learning task was discovered rather recently in the work of \cite{brukhim2022characterization}. The analogous object turns out to be a quantity known as the DS dimension $d_{\mathrm{DS}}$ \cite{daniely2014optimal}\footnote{We formally define all the relevant combinatorial quantities in \Cref{sec:preliminaries}.}. Given this, the right answer for the optimal sample complexity of multiclass classification, similar to \eqref{eqn:binary-optimal-sample-complexity} for binary classification, would appear to be
\begin{align}
    \label{eqn:multiclass-optimal-sample-complexity}
    \Theta\left(\frac{d_{\mathrm{DS}}+\log(1/\delta)}{\eps}\right).  
\end{align}
Unfortunately, a precise and provable characterization of this form had proven to be elusive. In their work showing that the DS dimension characterizes multiclass learnability, \cite{brukhim2022characterization} showed an upper bound of
\begin{align}
    \label{eqn:ds-1.5-bound}
    \widetilde{O}\left(\frac{d_{\mathrm{DS}}^{1.5}+\log(1/\delta)}{\eps}\right)  
\end{align}
on the sample complexity, where the $\widetilde{O}$ notation hides polylogarithmic factors. \cite{hanneke2024improved} got rid of some of these factors, but the $d_{\mathrm{DS}}^{1.5}$ scaling persisted. On the lower bound side, \cite{hanneke2024improved} showed a sample complexity lower bound of
\begin{align}
    \label{eqn:multiclass-sample-complexity-lower-bound}
    \Omega\left(\frac{d_{\mathrm{DS}}+\log(1/\delta)}{\eps}\right).
\end{align}
It had been widely conjectured that the bound above is optimal, and that the additional $\sqrt{d_{\mathrm{DS}}}$ term in \eqref{eqn:ds-1.5-bound}, along with any log factors, are extraneous. We note that the state of affairs was similar in the agnostic multiclass learning as well.

In a separate line of attack, \cite{aden2023optimal} showed a qualitatively different upper bound on the sample complexity of learning a multiclass hypothesis class $\mcH$ in the realizable setting, which depends on a combinatorial object known as the \textit{maximum density function} $\mu_\mcH(\cdot)$. In particular, \cite{aden2023optimal} showed that with $n$ samples, with probability $1-\delta$, it is possible to achieve error at most
\begin{align}
    \label{eqn:muh-upper-bound}
    O\left(\frac{\lceil\mu_\mcH(n)\rceil+\log(1/\delta)}{n}\right).
\end{align}
The attractive aspect about this bound is that it is completely devoid of any log factors, and has a linear dependence on $\mu_\mcH(n)$. While $\mu_\mcH(n)$ can in general be a non-decreasing function in $n$, a promising avenue to derive the optimal sample complexity bound in \eqref{eqn:multiclass-optimal-sample-complexity} would be to show that the function $\mu_\mcH(n)$ is uniformly upper-bounded by the DS dimension $d_{\mathrm{DS}}$. Indeed, this was precisely a conjecture formulated by \cite{daniely2014optimal} in their original work that defined the DS dimension.

\begin{conjecture*}[\cite{daniely2014optimal}]
    \label{conjecture:main}
    Let $\mcH \subseteq \mcY^\mcX$ be a hypothesis class having DS dimension $d_{\mathrm{DS}}$. There is an absolute constant $c > 0$ such that for any $n > 0$, $\mu_{\mcH}(n) \le c \cdot d_{\mathrm{DS}}$.%
\end{conjecture*}

There was reason to believe that the conjecture above is true, since in the special case of $|\mcY|=2$, the desired bound holds, by a beautiful result due to \cite{haussler1994predicting}. In addition, a weaker upper bound of the form $c \cdot d_{\mathrm{DS}} \cdot \log(|\mcY|)$ also holds \cite{daniely2014optimal}. Nevertheless, the conjecture in full generality has remained open since being proposed, and has resisted several natural proof approaches. Given the context of \cite{aden2023optimal}'s results, a positive resolution of the conjecture is especially meaningful, since it would establish the optimal sample complexity of multiclass learning.

\subsection{Results}
\label{sec:results}

Building upon a recent structural breakthrough by \cite{hanneke2026optimal}, our main result is a positive resolution of the longstanding conjecture of \cite{daniely2014optimal}, helping bypass the $d_{\mathrm{DS}}^{1.5}$ barrier on the sample complexity of multiclass learning. In fact, we show a more general result: for every $\ell \ge 1$, $\lceil\muh(n)\rceil \le \dds$, where $\muh(n)$ and $\dds$ are the maximum \textit{$\ell$-density} function and the \textit{$\ell$-DS} dimension respectively. These quantities are parameterized versions of $\mu_{\mcH}(n)$ and $d_{\mathrm{DS}}$ that were mentioned above, and satisfy that $\mu^1_\mcH(n)=\mu_\mcH(n)$ and $d^1_{\mathrm{DS}}=d_{\mathrm{DS}}$.

\begin{restatable}[Density upper-bounded by DS]{theorem}{TheoremMain}
    \label{thm:main-upper-bound}
    Let $\mcH \subseteq \mcY^\mcX$ be a hypothesis class having $\ell$-DS dimension $\dds$. For all integers $\ell \ge 1$ and $n >0$,
    \begin{align}
        \lceil\muh(n)\rceil \le \dds.
    \end{align}
\end{restatable}
Several results follow from \Cref{thm:main-upper-bound}, which we proceed to list down.

\newcolumntype{C}[1]{>{\centering\arraybackslash}m{#1}}
\setlength{\tabcolsep}{-5pt}
\renewcommand{\arraystretch}{1.8}

\begin{table}[t]
    \centering
    \begin{minipage}{0.48\textwidth}
    \centering
    \begin{tabular}{|C{3.7cm}|C{5cm}|}
    \hline
    \multicolumn{2}{|c|}{\textbf{Realizable}} \\
    \hline
    Reference & Sample Complexity \\
    \hline
    \cite{brukhim2022characterization, hanneke2024improved} & $\widetilde{O}\left(\frac{d_{\mathrm{DS}}^{1.5}+\log(1/\delta)}{\eps}\right)$ \\
    \hline
    \cite{aden2023optimal} & $O\left(\frac{\mu_\mcH(\cdot)+\log(1/\delta)}{\eps}\right)$ \\
    \hline
    \cite{hanneke2024improved} & $\Omega\left(\frac{d_{\mathrm{DS}}+\log(1/\delta)}{\eps}\right)$ \\
    \hline
    \end{tabular}
    
    \vspace{3pt}
    
    \begin{tabular}{|C{3.7cm}|C{5cm}|}
    \hline
    \makecell{\textbf{This work}\\[2pt](\Cref{corollary:multiclass-optimal-sample-complexity})}
    &
    \raisebox{-0.1\height}{$O\left(\frac{d_{\mathrm{DS}}+\log(1/\delta)}{\eps}\right)$} \\
    \hline
    \end{tabular}
    \end{minipage}
    \hfill
    \begin{minipage}{0.48\textwidth}
    \centering
    \begin{tabular}{|C{3.7cm}|C{5cm}|}
    \hline
    \multicolumn{2}{|c|}{\textbf{Agnostic}} \\
    \hline
    Reference & Sample Complexity \\
    \hline
    \cite{brukhim2022characterization} & $\widetilde{O}\left(\frac{d_{\mathrm{DS}}^{1.5}+\log(1/\delta)}{\eps^2}\right)$ \\
    \hline
    \cite{cohen2025sample} & $\widetilde{O}\left(\frac{d_{\mathrm{real}}}{\eps}+\frac{d_{\mathrm{Nat}}+\log(1/\delta)}{\eps^2}\right)$ \\
    \hline
    \cite{cohen2025sample} & $\Omega\left(\frac{d_{\mathrm{real}}}{\eps}+\frac{d_{\mathrm{Nat}}+\log(1/\delta)}{\eps^2}\right)$ \\
    \hline
    \end{tabular}
    
    \vspace{3pt}
    
    \begin{tabular}{|C{3.7cm}|C{5cm}|}
    \hline
    \makecell{\textbf{This work}\\[2pt](\Cref{corollary:agnostic-multiclass-near-optimal-sample-complexity})}
    &
    \raisebox{-0.1\height}{$\widetilde{O}\left(\frac{d_{\mathrm{DS}}}{\eps}+\frac{d_{\mathrm{Nat}}+\log(1/\delta)}{\eps^2}\right)$} \\
    \hline
    \end{tabular}
    \end{minipage}
    \caption{Sample Complexity Bounds for Multiclass Learning}
    \label{table:multiclass-learning-rates}
\end{table}

\paragraph{Multiclass Learning.} We begin with the consequences for multiclass learning. Recall from \eqref{eqn:ds-1.5-bound} that \cite{brukhim2022characterization} had shown an $\widetilde{O}\left(\frac{d_{\mathrm{DS}}^{1.5}+\log(1/\delta)}{\eps}\right)$ upper bound on the sample complexity of this task in the realizable setting. While their bound was improved upon by \cite{hanneke2024improved}, who shaved off several logarithmic factors, the $d_{\mathrm{DS}}^{1.5}$ dependence persisted. As the first main corollary of \Cref{thm:main-upper-bound}, we obtain the optimal sample complexity for realizable multiclass learning.

\begin{corollary}[Optimal Sample Complexity for Multiclass Learning]
    \label{corollary:multiclass-optimal-sample-complexity}
    Let $\mcH \subseteq \mcY^\mcX$ be a hypothesis class having DS dimension $d_{\mathrm{DS}}$. There exists a learning algorithm $\mcA$, such that for any distribution $\mcD$ over $\mcX \times \mcY$ realizable by $\mcH$, and any $\eps, \delta \in (0,1)^2$, with probability at least $1-\delta$ over a sample $S \sim \mcD^m$, where
    \begin{align*}
        m \ge  9.64 \left(\frac{d_{\mathrm{DS}}+\log(2/\delta)}{\eps}\right),
    \end{align*}
    it holds that $\err_\mcD(\widehat{f}_S) \le \eps$, 
    where $\widehat{f}_S=\mcA(S)$.
\end{corollary}
The result above follows by directly plugging in the upper bound on $\lceil \mu_\mcH(n) \rceil$ from \Cref{thm:main-upper-bound} into Theorem 2.2 of \cite{aden2023optimal}. The learning algorithm achieving the guarantee is a majority vote over \textit{one-inclusion graph predictors} (see \Cref{def:oig}) trained on increasing prefixes of $S$, and we give the necessary proof details in \Cref{sec:realizable-list-learning}. Together with the lower bound of $\Omega\left(\frac{d_{\mathrm{DS}}+\log(1/\delta)}{\eps}\right)$ on the sample complexity given by \cite[Theorem 2.5]{hanneke2024improved}, we thus confirm that the optimal sample complexity of multiclass learning is $\Theta\left(\frac{d_{\mathrm{DS}}+\log(1/\delta)}{\eps}\right)$.

We next consider multiclass learning in the \textit{agnostic} setting, where the distribution $\mcD$ over $\mcX \times \mcY$ may be arbitrary, and not necessarily realizable by $\mcH$. Here, %
the objective of the learning algorithm is to achieve expected error comparable to $\inf_{h \in \mcH}\err_{\mcD}(h)$. For this setting, \cite{brukhim2022characterization} showed an upper bound of $\widetilde{O}\left(\frac{d_{\mathrm{DS}}^{1.5}+\log(1/\delta)}{\eps^2}\right)$ on the sample complexity. Recently, \cite{cohen2025sample} improved this to a near-tight sample complexity upper bound of $\widetilde{O}\left(\frac{d_{\mathrm{real}}}{\eps}+\frac{d_{\mathrm{Nat}}+\log(1/\delta)}{\eps^2}\right)$, %
where $d_{\mathrm{real}}$ is the optimal sample complexity of learning $\mcH$ in the \textit{realizable} setting, $d_{\mathrm{Nat}}$ is the \textit{Natarajan} dimension of $\mcH$ \cite{natarajan1988two,natarajan1989learning} (which satisfies that $d_{\mathrm{Nat}} \le d_{\mathrm{DS}}$), and the $\widetilde{O}$ notation hides polylogarithmic factors in $d_{\mathrm{Nat}}, d_{\mathrm{real}}$ and $1/\eps$. Note crucially that this bound was stated in terms of $d_{\mathrm{real}}$, which, at the time, was unknown. As a direct consequence of \Cref{corollary:multiclass-optimal-sample-complexity}, we get the following more precise upper bound for the learning algorithm given by \cite{cohen2025sample}.

\begin{corollary}[Optimal Sample Complexity for Agnostic Multiclass Learning]
    \label{corollary:agnostic-multiclass-near-optimal-sample-complexity}
    Let $\mcH \subseteq \mcY^\mcX$ be a hypothesis class having DS dimension $d_{\mathrm{DS}}$ and Natarajan dimension $d_{\mathrm{Nat}}$. There exists a learning algorithm $\mcA$, such that for any distribution $\mcD$ over $\mcX \times \mcY$, and any $\eps, \delta \in (0,1)^2$, with probability at least $1-\delta$ over a sample $S \sim \mcD^m$, where
    \begin{align*}
        m = \widetilde{O}\left(\frac{d_{\mathrm{DS}}}{\eps}+\frac{d_{\mathrm{Nat}}+\log(1/\delta)}{\eps^2}\right),
    \end{align*}
    it holds that $\err_\mcD(\widehat{f}_S) \le \inf_{h \in \mcH} \err_{\mcD}(h) + \eps$,
    where $\widehat{f}_S=\mcA(S)$.
\end{corollary}

Combined with the agnostic sample complexity lower bound of $\Omega\left(\frac{d_{\mathrm{real}}}{\eps}+\frac{d_{\mathrm{Nat}}+\log(1/\delta)}{\eps^2}\right)$ also shown by \cite{cohen2025sample}, we thus get that the optimal sample complexity of agnostic multiclass learning (upto log factors) is $\widetilde{\Theta}\left(\frac{d_{\mathrm{DS}}}{\eps}+\frac{d_{\mathrm{Nat}}+\log(1/\delta)}{\eps^2}\right)$. \Cref{table:multiclass-learning-rates} summarizes the above results. %

\renewcommand{\arraystretch}{1.8}

\begin{table}[t]
    \centering
    \begin{minipage}{0.48\textwidth}
    \centering
    \begin{tabular}{|C{3.7cm}|C{5cm}|}
    \hline
    \multicolumn{2}{|c|}{\textbf{Realizable}} \\
    \hline
    Reference & Sample Complexity \\
    \hline
    \cite{charikar2023characterization} & $\widetilde{O}\left(\frac{\ell^6\left(\dds\right)^{1.5}+\log(1/\delta)}{\eps}\right)$ \\
    \hline
    \cite{brukhim2023multiclass} & $\widetilde{O}\left(\frac{\ell^4\left(\dds\right)^{5}+\log(1/\delta)}{\eps}\right)$ \\
    \hline
    \cite{hanneke2026optimal} & $\widetilde{O}\left(\frac{\ell\left(\dds\right)^{1.5}+\ell\log(1/\delta)}{\eps}\right)$  \\
    \hline
    \cite{hanneke2024improved} & $\Omega\left(\frac{\dds+\log(1/\delta)}{\ell\eps}\right)$  \\
    \hline
    \end{tabular}
    
    \vspace{3pt}
    
    \begin{tabular}{|C{3.7cm}|C{5cm}|}
    \hline
    \makecell{\textbf{This work}\\[2pt](\Cref{thm:realizable-list-learning-lower-bound})}
    &
    \raisebox{-0.1\height}{$\Omega\left(\frac{\dds+\log(1/\delta)}{\eps}\right)$} \\
    \hline
    \makecell{\textbf{This work}\\[2pt](\Cref{thm:list-optimal-sample-complexity})}
    &
    \raisebox{-0.1\height}{$O\left(\frac{\dds+\log(1/\delta)}{\eps}\right)$} \\
    \hline
    \end{tabular}
    \end{minipage}
    \hfill
    \begin{minipage}{0.48\textwidth}
    \centering
    \begin{tabular}{|C{3.7cm}|C{5cm}|}
    \hline
    \multicolumn{2}{|c|}{\textbf{Agnostic}} \\
    \hline
    Reference & Sample Complexity \\
    \hline
    \cite{charikar2023characterization} & $\widetilde{O}\left(\frac{\ell^6\left(\dds\right)^{1.5}+\log(1/\delta)}{\eps^2}\right)$ \\
    \hline
    \cite{hanneke2024improved} & $\Omega\left(\frac{\dds+\log(1/\delta)}{\ell\eps}\right)$ \\
    \hline
    \end{tabular}
    
    \vspace{3pt}
    
    \begin{tabular}{|C{3.7cm}|C{5cm}|}
    \hline
    \makecell{\textbf{This work}\\[2pt](\Cref{thm:agnostic-list-learning-lower-bound})}
    &
    \raisebox{-0.1\height}{$\Omega\left(\frac{\dds}{\eps}+\frac{\ell d^\ell_{\mathrm{Nat}}+\log(1/\delta)}{\eps^2}\right)$} \\
    \hline
    \makecell{\textbf{This work}\\[2pt](\Cref{thm:agnostic-list-improved-sample-complexity})}
    &
    \raisebox{-0.1\height}{$\widetilde{O}\left(\frac{\dds}{\eps}+\frac{\ell d^\ell_{\mathrm{Nat}}+\log(1/\delta)}{\eps^2}\right)$} \\
    \hline
    \end{tabular}
    \end{minipage}
    \caption{Sample Complexity Bounds for List Learning}
    \label{table:list-learning-rates}
\end{table}

\paragraph{Optimal Universal Rates for Multiclass Learning} \Cref{corollary:multiclass-optimal-sample-complexity} together with Theorem 66 in \cite{hanneke2023universal} implies that every hypothesis class over a countable label space $\mcY$ that has no infinite DSL tree admits the universal learning rate $\frac{1}{n}$. This improves the $\frac{\log^2n}{n}$ rate in Theorem 8 in \cite{hanneke2023universal}, which therefore sharpens to an exact trichotomy for multiclass classes: the optimal universal rate is exponential when there is no infinite Littlestone tree, is $\frac{1}{n}$ when there is an infinite Littlestone tree but no infinite DSL tree, and learning requires arbitrarily slow rates otherwise. This resolves the first open question asked in their Section 1.3.

\paragraph{Multiclass Sample Compression.} The learner with optimal sample complexity from \Cref{corollary:multiclass-optimal-sample-complexity} can also be converted into a \textit{sample compression scheme}. The objective in sample compression is to compress any sample $S$ realizable by a hypothesis class $\mcH$ to a subsample $S'$, such that one can \textit{reconstruct} labels on all points in the original sample $S$ simply by looking at $S'$. By a standard application of the minimax theorem to the learner from \Cref{corollary:multiclass-optimal-sample-complexity}, one obtains a sample compression scheme that validly compresses a sample of size $m$ to a subsample of size $O(d_{\mathrm{DS}} \log m)$. In fact, this compression scheme is optimal: there exist hypothesis classes for which the maximum amount of compression possible is $\Omega(d_{\mathrm{DS}} \log m)$. This follows by sourcing the lower bound argument in \cite{pabbaraju2024multiclass} from the recent optimal refutation of the Alon-Saks-Seymour conjecture \cite{alon2022theory,pabbaraju2026optimal}, and suitably ``productizing'' it so that it may be applicable for any DS dimension value. The optimal amount of sample compression generically possible for multiclass classes is thus $\Theta(d_{\mathrm{DS}} \log m)$. We give the necessary details in \Cref{sec:multiclass-sample-compression}.

\begin{restatable}[Optimal Multiclass Sample Compression]{theorem}{SampleCompression}
    \label{thm:sample-compression}
    For any class $\mcH$ having DS dimension $d_{\mathrm{DS}}$, there exists a sample compression scheme that compresses samples of size $m$ to subsamples of size $O(d_{\mathrm{DS}} \log m)$. Furthermore, for any $d_{\mathrm{DS}}$, there exists a hypothesis class $\mcH$ having DS dimension $d_{\mathrm{DS}}$ for which any valid sample compression scheme requires compressing certain samples of size $m$ to size $\Omega(d_{\mathrm{DS}} \log m)$.
\end{restatable}

\paragraph{List Learning.} We now turn to the related task of \textit{list} learning. In list learning, an algorithm is permitted to output a list hypothesis $\mu$, which maps any point $x$ to a list of $\ell$ labels instead of a single label. The error of the list hypothesis is measured as $\err_{\mcD}(\mu)=\Pr_{(x,y) \sim \mcD}[y \notin \mu(x)]$. In this way, the algorithm has a more relaxed goal: so long as the list of labels it outputs merely \textit{contains} the target label, it does not get penalized. When $\ell=1$, this task is equivalent to standard multiclass learning. When $\ell > 1$, list learning has a characterization analogous to multiclass learning in terms of the \textit{$\ell$-DS dimension} $\dds$ \cite{charikar2023characterization}, whose finiteness is both sufficient and necessary for the task. Concretely, in the realizable case, \cite{hanneke2024improved} showed a sample complexity lower bound of $\Omega\left(\frac{\dds+\log(1/\delta)}{\ell \eps}\right)$ for list learning. Very recently (and in independent work), this lower bound was improved by \cite{li2026optimisticratesmulticlasspac} to $\Omega\left(\frac{\dds}{\eps}\right)$. In fact, a more precise lower bound of $\Omega\left(\frac{\dds+\log(1/\delta)}{ \eps}\right)$ holds, which we prove in \Cref{sec:realizable-list-learning-lower-bound} for completeness.

On the upper bound side, \cite{charikar2023characterization} showed that the sample complexity is at most $\widetilde{O}\left(\frac{\ell^6(\dds)^{1.5}+\log(1/\delta)}{\eps}\right)$, whereas \cite{brukhim2023multiclass} showed an (incomparable) upper bound of $\widetilde{O}\left(\frac{\ell^4(\dds)^{5}+\log(1/\delta)}{\eps}\right)$. Very recently, \cite{hanneke2026optimal} proved an upper bound of $\widetilde{O}\left(\frac{\ell(\dds)^{1.5}+\ell\dds\log(1/\delta)}{\eps}\right)$, thus significantly improving on the dependence on list size $\ell$. Note that similar to the multiclass learning barrier, a barrier of $(\dds)^{1.5}$ also persists for list learning; our \Cref{thm:main-upper-bound} enables getting rid of the extra $\sqrt{\dds}$ factor, as well as all log factors.

\begin{restatable}[Optimal Sample Complexity for List Learning]{theorem}{RealizableList}
    \label{thm:list-optimal-sample-complexity}
    Let $\mcH \subseteq \mcY^\mcX$ be a hypothesis class having $\ell$-DS dimension $\dds$, for any $\ell \ge 1$. There exists a (randomized) list learning algorithm $\mcA$, such that for any distribution $\mcD$ over $\mcX \times \mcY$ realizable by $\mcH$, and any $\eps, \delta \in (0,1)^2$, with probability at least $1-\delta$ over a sample $S \sim \mcD^m$, where
    \begin{align*}
        m =  O \left(\frac{\dds+\log(1/\delta)}{\eps}\right),
    \end{align*}
    it holds that $\err_\mcD(\widehat{\mu}_S) \le \eps$, where $\widehat{\mu}_S=\mcA(S)$.
\end{restatable}

The bound above is thus optimal in its dependence on all involved parameters, including the list size $\ell$. Recall that for realizable multiclass learning, we were able to directly invoke the previous results of \cite{aden2023optimal}. As it turns out, their framework is general enough to also capture list learning; however, a direct application only gives a deterministic list learner whose sample complexity is suboptimal by a multiplicative $\ell$ factor (\Cref{corollary:list-near-optimal-sample-complexity}). Nevertheless, using randomization and a holdout validation set as in \cite{haussler1994predicting}, we can construct a list learner that shaves this factor and has the optimal sample complexity.\footnote{In the first version of this work, we had only stated the sample complexity upper bound for the deterministic learner from \Cref{corollary:list-near-optimal-sample-complexity}, which was suboptimal by a factor of $\ell$. This factor gets shaved by using randomization and the standard trick of using a small validation sample \cite{haussler1994predicting}. We note that this improvement was independently also observed by \cite{dewasurendra2026select}.} We provide the necessary details in \Cref{sec:realizable-list-learning}.

Finally, we consider the task of \textit{agnostic} list learning. Here, the objective again is to have small excess error compared to $\inf_{h \in \mcH}\err_\mcD(h)$. %
To our knowledge, the previous best sample complexity upper bound for this task is that of $\widetilde{O}\left(\frac{\ell^6(\dds)^{1.5}+\log(1/\delta)}{\eps^2}\right)$ given by \cite{charikar2023characterization}. We improve upon this by giving a new agnostic list learner that has optimal agnostic sample complexity. %
The proof details are given in \Cref{sec:agnostic-list-learning}.

\begin{restatable}[Optimal Sample Complexity for Agnostic List Learning]{theorem}{AgnosticList}
    \label{thm:agnostic-list-improved-sample-complexity}
    Let $\mcH \subseteq \mcY^\mcX$ be a hypothesis class having $\ell$-DS dimension $\dds$ and $\ell$-Natarajan dimension $d^\ell_{\mathrm{Nat}}$. There exists a list learning algorithm $\mcA$, such that for any distribution $\mcD$ over $\mcX \times \mcY$, and any $\eps, \delta \in (0,1)^2$, with probability at least $1-\delta$ over a sample $S \sim \mcD^m$, where
    \begin{align*}
        m = \widetilde{O}\left(\frac{\dds}{\eps}+\frac{\ell d^\ell_{\mathrm{Nat}}+\log(1/\delta)}{\eps^2}\right),
    \end{align*}
    it holds that $\err_\mcD(\widehat{\mu}_S) \le \inf_{h \in \mcH} \err_{\mcD}(h) + \eps$, where $\widehat{\mu}_S=\mcA(S)$.
\end{restatable}
The dependence on the $\ell$-DS dimension $\dds$ above is optimal upto log factors (the $\ell$-Natarajan dimension satisfies $d^\ell_{\mathrm{Nat}} \le \dds$), since our $\Omega\left(\frac{\dds+\log(1/\delta)}{ \eps}\right)$ lower bound on sample complexity from the realizable setting carries over to the agnostic setting. For $\ell >1$, proving necessity of the $1/\eps^2$ terms %
in the sample complexity turns out to be quite challenging. At a high level, a list predictor seems to have a ``hedging'' advantage when competing with only the best \textit{single} hypothesis $h \in \mcH$, which makes the standard lower bound strategy %
from the $\ell=1$ setting not directly work. In comparison, one would expect these terms to be necessary in the case where the algorithm has to compete with the best \textit{list} hypothesis that can be formed from hypotheses in $\mcH$, as also recently confirmed by \cite{li2026optimisticratesmulticlasspac}.

Nevertheless, in \Cref{thm:agnostic-list-learning-lower-bound}, we show that the agnostic sample complexity of $\ell$-list learning is lower bounded both by $\Omega(\log(1/\delta)/\eps^2)$ and $\Omega(\ell \dnat/\eps^2)$, thus establishing a combined lower bound of 
$$
\Omega\left(\frac{\dds}{\eps}+\frac{\ell \dnat + \log(1/\delta)}{\eps^2}\right).
$$
The $\ell \dnat$ term is perhaps surprising, since the list size $\ell$ does not explicitly show up in the optimal realizable sample complexity. Our lower bound also confirms that the upper bound in \Cref{thm:agnostic-list-improved-sample-complexity} is optimal in the dependence on \textit{all} involved parameters, modulo polylogarithmic factors. \Cref{table:list-learning-rates} summarizes all the results above. %

We remark that both the upper and lower bounds above for the agnostic setting require several novel technical arguments. Lastly, we explicitly note that all our results hold even when the label space $\mcY$ is infinite.

%% file: proof-sketches.tex
\section{Proof Sketches}
\label{sec:sketches}

\subsection{Upper-bounding Maximum Density by DS}
\label{sec:main-theorem-sketch}

Our proof of \Cref{thm:main-upper-bound} builds upon the elegant algebraic characterization of multiclass hypothesis classes derived very recently by \cite{hanneke2026optimal}. We first formally define the key technical objects from \cite{hanneke2026optimal} that enable our result. Let $W \subseteq [k]^n$ be a hypothesis class, where $k < \infty$. We can think of each $w \in W$ as a vector $(w_1,\dots,w_n) \in [k]^n$. Associate $W$ with the \textit{vector space} of functions $\mcV_W = \{f:W \to \R\}$, equipped with pointwise addition and scalar multiplication. We can interpret the members of this vector space equivalently as vectors in $\R^{|W|}$. Then, for any $s \ge 0$ and $\ell \ge 1$, consider the following special subset of $\mcV_W$ comprising of bounded-degree, bounded-support monomials\footnote{We identify every member of $\mcM^\ell_s(W)$ with a fixed monomial representation $w_1^{\alpha_1} w_2^{\alpha_2}\dots w_n^{\alpha_n}$.}:
\begin{align}
    \label{eqn:monomial-set-def}
    \mcM^\ell_s(W) :=&\left\{f : W \to \R ~\Big\vert~ f(w)=w_1^{\alpha_1} w_2^{\alpha_2}\dots w_n^{\alpha_n} \; \forall w \in W, \right. \nonumber \\
    &\hspace{2.4cm}\left. \alpha_i \in \{0,1,\dots,k-1\}
    \; \forall i \in [n], \right. \nonumber \\
    &\hspace{2.4cm}\left. \left|\{i \in [n]:\alpha_i \ge \ell\}\right| \le s \right\}.
\end{align}

With the $\R^{|W|}$ viewpoint of $\mcV_W$, the above definition says that if we index the coordinates of a vector by $w \in W$, then $\mcM^\ell_s(W)$ comprises of those vectors where the entry in the $w^{\text{th}}$ coordinate is given by $w_1^{\alpha_1} w_2^{\alpha_2}\dots w_n^{\alpha_n}$, where the degree $\alpha_i$ of each $w_i$ is at most $k-1$, and at most $s$ values of $\alpha_i$ are greater than or equal to $\ell$.

The breakthrough insight in  \cite{hanneke2026optimal} is that the set of monomials $\mcM^\ell_s(W)$ spans $\mcV_W$. While this structural characterization is contained in the proof of Theorem 2.1 in \cite{hanneke2026optimal}, we state this result as a standalone lemma.

\begin{lemma}[Spanning Lemma \cite{hanneke2026optimal}]
    \label{lemma:spanning}
    Let $W \subseteq [k]^n$ be a hypothesis class, where $k < \infty$. For all integers $\ell \ge 1$ and $s \ge \dds(W)$, $\mcM^\ell_{s}(W)$ spans $\mcV_W$.
\end{lemma}
We build upon the linear-algebraic characterization in the spanning lemma to show our main \Cref{thm:main-upper-bound}. %
Here, we sketch the proof when $\ell=1$ and $\mcY = [k]$ for some finite $k$, which is simpler and captures the main ideas. The core intuition is the following: any monomial $w_1^{\alpha_1}\dots w_n^{\alpha_n}$ in the spanning set, which has $\alpha_i=0$, is \textit{constant} on every edge in direction $i$ in the one-inclusion graph of the hypothesis class. Therefore, one can consider the \textit{subspace} of functions in $\mcV_W$ that have this defining property, and has dimension \textit{equal} to the number of edges in direction $i$. Any monomial within the spanning set that has $\alpha_i=0$ belongs to this subspace. But crucially, the total number of \textit{linearly independent} monomials in the spanning set that have $\alpha_i=0$ must \textit{lower-bound} the dimension of the subspace, unlocking the key inequality.

More concretely, let $W \subseteq [k]^n$ be a hypothesis class with DS dimension $d_{\mathrm{DS}}$. In this case, for $s=d_{\mathrm{DS}}$, each monomial in the spanning set $\mcM^1_s(W)$ is a monomial in at most $s$ ``active'' coordinates $i$, where $\alpha_i > 0$. Given that $\mcM^1_s(W)$ spans $\mcV_W$, choose a basis $\mcB \subseteq \mcM^1_s(W)$. Since $\mcV_W$ has dimension $|W|$, the basis $\mcB$ has exactly $|W|$ monomials. The key definition now is that of a \textit{direction-wise subspace} $U_i$ for every direction $i \in [n]$. This subspace comprises of all the functions that are constant when restricted to an edge in direction $i$ in the one-inclusion graph $\mcG(W)$. As a toy example with $n=2$, consider $W=\{(2,1), (3,1), (4,2), (5,2), (6,2)\}$. Then, $U_1$ comprises of all the vectors in $\R^5$ of the form $(a,a,b,b,b)$. We can readily see that the dimension of $U_i$ is exactly the number of different behaviors realized by $W$ on $[n] \setminus \{i\}$, or equivalently, the number of distinct edges in $\mcG(W)$ in direction $i$ --- let us call this number $N_i$.

Now, consider the set of monomials $\mcB^0_i$ in the basis $\mcB$ for which the coordinate $i$ is \textit{not} active (i.e., $\alpha_i=0$), and let $\mcB^1_i = \mcB \setminus \mcB^0_i$. Observe that every monomial in $\mcB^0_i$ belongs to $U_i$, since it evaluates to the same number on the members of any edge in direction $i$. However, since all the monomials in $\mcB^0_i$ are linearly independent, it must hold that 
\begin{align*}
    |\mcB^0_i| \le \dim(U_i)=N_i \quad \implies \quad |W| - N_i \le |W|-|\mcB^0_i| = |\mcB^1_i|.
\end{align*}
Now let us sum this last inequality over all $i \in [n]$. The term on the left is equal to
\begin{align*}
    \sum_{i=1}^n (|W|-N_i) = \sum_{i \in [n]}\sum_{e_{i,a} \in \mcG(W)} (|e_{i,a}|-1) = \sum_{e 
    \in \mcG(W)}(|e|-1) = |W| \cdot \dens^1(W).
\end{align*}
On the other hand, summing the term on the right, namely $\sum_{i=1}^n |\mcB^1_i|$, simply counts, for every monomial $b \in \mcB$, the number of coordinates $i$ at which it is active. Since $\mcB \subseteq \mcM^1_s(W)$, and every monomial in the spanning set has at most $s$ active coordinates, we get that this sum is at most $s \cdot |\mcB| = s|W| = d_{\mathrm{DS}}|W|$. Dividing by $|W|$ on both sides gives the inequality in the theorem statement.

\subsection{Realizable List Learning}
\label{sec:realizable-sketches}

\paragraph{Upper Bound.} For the upper bound for realizable list learning, we first show how the machinery of \cite{aden2023optimal} can be directly combined with \Cref{thm:main-upper-bound} to give a deterministic list learner that has sample complexity $O\left(\frac{\ell(\dds+\log(1/\delta))}{\eps}\right)$. Concretely, the work of \cite{aden2023optimal} shows that the average error of predictors trained on increasing prefixes of the training data is small, so long as the individual predictors are symmetric in their data and satisfy a \textit{leave-one-out} error guarantee. Moreover, this average error is upper bounded by the leave-one-out error of the individual predictors. Now, the list version of the one-inclusion graph algorithm of \cite{charikar2023characterization} is symmetric, and also satisfies that its leave-one-out error is at most the largest \textit{outdegree} of any vertex in an \textit{oriented} one-inclusion graph. Moreover, \cite{hanneke2026optimal} show that any finite one-inclusion graph may be oriented in a manner where every vertex has outdegree at most the maximum $\ell$-density of the hypothesis class, and this guarantee extends to infinite graphs via a compactness argument (see \Cref{lemma:compactness}). Since \Cref{thm:main-upper-bound} upper-bounds the maximum $\ell$-density by $\dds$, we get that the leave-one-out error of the one-inclusion graph predictor is at most $\dds$. Thus, the average error of these predictors applied to increasing prefixes scales as $\dds$. Taking a $\topk$ vote over the predictors then yields a learner whose sample complexity is $O(\ell \dds/\eps)$.

The extra $\ell$ factor above arises because of the $\topk$ vote; in order to get rid of it, we use randomization along with validation on a holdout set, based on the confidence amplification procedure of \cite{haussler1994predicting}. Namely, since the generalization bound of \cite{aden2023optimal} ensures that the average error of predictors trained on increasing prefixes of a training sample of size $O(\dds)$ is small, we know that a decent fraction of these predictors achieve small error. So, if we randomly sample a few prefixes and obtain predictors trained on these prefixes, some predictor will have small error with decent probability. In order to identify this predictor, we obtain an additional nominal sample for validation. A standard Chernoff bound then ensures that we select an accurate predictor, with the desired optimal sample complexity.

\paragraph{Lower Bound.} The arguments for the matching lower bound for realizable list learning are essentially identical to those used by \cite{hanneke2024improved}. However, we apply these arguments to classes that satisfy a stronger ``label-richness'' property, in addition to having DS dimension $\dds$. In particular, we require that some set of size $\dds$ is DS-shattered \textit{not only} at a scale of $\ell$, but also at a larger scale of $2\ell-1$ (i.e., $d_{\mathrm{DS}}^{2\ell-1} \ge \dds$). At a high level, this ensures that any predicted list of $\ell$ labels leaves available \textit{half} the amount of freedom for a label at any point in a shattered set (as opposed to leaving only an $O(1/\ell)$ amount of freedom). With this additional assumption on the class, we can carry out the analysis of \cite{hanneke2024improved} more or less identically, and improve their lower bound by a factor of $\ell$.

\subsection{Agnostic List Learning}
\label{sec:agnostic-sketches}

\paragraph{Upper Bound.} For the agnostic setting, we follow the high-level proof strategy of \cite{cohen2025sample}, which comprises of three main steps: (1) Approximation of the class $\mcH$ by a finite list cover $\mcF$, which ensures that for every $h \in \mcH$, there is a corresponding list hypothesis $\mu_h \in \mcF$ that approximates its behavior, (2) Multiplicative Weights on the cover $\mcF$, which shows that a ``menu'' given by a fixed subset of $\approx \log(|\mcF|)$ many members of $\mcF$ approximates the behavior of any given list hypothesis in $\mcF$, and finally, (3) Learning a classifier when the target label is bounded within the finite menu of labels.

Each of the three steps above uses a fresh labeled sample from the data distribution. The first two steps together require only $\widetilde{O}(\dds/\eps)$ many samples, but make the problem of dealing with an infinite label space tractable, by effectively trapping the label assigned by any $h \in \mcH$ (and in particular, the benchmark hypothesis $h^\star$ in the agnostic objective) at any $x$ within a small menu $\nu(x)$ of %
possible labels. These steps also ensure that going forward, in order to obtain the final agnostic learning guarantee for any candidate predictor $f$, it suffices to only worry about the ``menu-conditional loss'' 
\begin{align*}
    \Ind[y \in \nu(x), y \neq f(x)],
\end{align*} 
which penalizes $f$ only when it makes a mistake at an $x$ whose label is in the  menu.
Together, this allows us to bound the total number of label patterns that we need to consider on a sample in terms of the $\ell$-Natarajan dimension, leading to deriving an agnostic learner in the third step that uses only $\widetilde{O}(\ell \dnat/\eps^2)$ additional samples. 

Steps (1) and (2) in our analysis parallel the corresponding steps in \cite{cohen2025sample} very closely. Namely, Step (1) derives a list sample compression scheme of size $\widetilde{O}(\dds)$ for samples realizable by $\mcH$ by applying a minimax argument to the one-inclusion graph based list learner for $\mcH$ \cite{charikar2023characterization}. This is followed by invoking a standard generalization bound for compression schemes, which shows that the finite space of possible reconstructions of the compression scheme forms an approximate list cover $\mcF$ for $\mcH$. Thereafter, Step (2) processes a fresh sample in an online fashion, while maintaining a probability distribution over $\mcF$. At each time step, it samples a list hypothesis from $\mcF$, and thereafter performs a multiplicative update on the weights, based on whether the label of the revealed data point is captured by the set of list hypotheses sampled so far. The analysis of \cite{cohen2025sample} applies verbatim here, and guarantees that the concatenation of the list hypotheses at the end of the online process forms an accurate finite menu of labels $\nu$ for each $x$.

This brings us to Step (3), where we wish to capitalize on the label-space reduction achieved in Steps (1) and (2), and derive an agnostic list learner. Here as well, we could follow \cite{cohen2025sample}, and use the machinery of agnostic sample compression schemes. However, a direct implementation of that methodology requires using $\widetilde{O}(\ell^4 \dnat/\eps^2)$ many samples, which is much larger than our target of $\widetilde{O}(\ell \dnat/\eps^2)$. Therefore, our implementation of Step (3) deviates significantly from \cite{cohen2025sample}, and we have to resort to several novel and fundamentally different tools. Concretely, we borrow inspiration from the online-to-batch analyses employed in \cite{alon2022theory} and \cite{moran2023list} to derive a \textit{fractional} $\ell$-list predictor $q$ that is supported on the menu $\nu(x)$ for every $x$. Namely, $q(x)$ assigns a probability to each label in $\nu(x)$, and the sum of these probabilities is at most $\ell$. The list predictor $q$ is obtained by processing the fresh sample for Step (3) in an online fashion, and computing a suitable \textit{shattering potential}. We show that this shattering potential endows the derived fractional list predictor with a small agnostic \textit{transductive} error. This allows us to use the results of \cite{dughmi25transductive}, which converts transductive learners in the agnostic setting to PAC learners. As a final caveat, this conversion only gives a fractional list PAC learner: we also give a novel rounding scheme that rounds the fractional list into a bona fide list of $\ell$ labels, while inheriting the error guarantees of the fractional learner. We believe that each component  used here could be of independent value in other applications.

We now flesh out Step (3) in more detail. Given a fresh labeled sample $S=\{(x_1,y_1),\dots,(x_n,y_n)\}$, and an unlabeled test point $x$, we first construct a deterministic fractional list predictor $Q(x; S)$, which outputs a fractional $\ell$-list supported on the menu $\nu(x)$. %
We obtain $Q(x; S)$  as follows: consider the entire unlabeled sequence $(x_1,\dots,x_n,x)$, in a fixed permutation $\pi$, to be 
    \begin{align*}
        x_{i_1}, x_{i_2},\dots,x_{i_{t}}, x, x_{i_{t+1}},\dots,x_{i_n}.
    \end{align*}
We will obtain a fractional $\ell$-list $Q_\pi$ supported on $\nu(x)$ by processing the points $x_{i_1}, x_{i_2},\dots,x_{i_{t}}$ \textit{preceding} $x$ in this permutation. To obtain $Q_\pi$, we will process $x_{i_1}, x_{i_2},\dots,x_{i_{t}}$ in an online fashion, where we will retain any $(x_{i_r}, y_{i_r})$ for which $y_{i_r} \in \nu(x_{i_r})$ independently with probability $\gamma$. Let $V_t$ denote the (random) version space induced by the menu-consistent labeled points that get retained as we process $x_{i_1}, x_{i_2},\dots,x_{i_{t}}$, and let $s(V_t)$ denote the number of subsets of the entire unlabeled sequence that are $\ell$-Natarajan-shattered by $V_t$. Now, for any label $y \in \nu(x)$, consider the subset $V_{t, y} \subseteq V_t$ that further restricts the label on $x$ to be $y$. Then, the mass that $Q_{\pi}$ assigns to $y$ is set to equal $\E[s(V_{t,y})]/\E[s(V_{t})]$, where the expectation is over the labeled samples preceding $x$ that get retained. Our final prediction $Q(x; S)$ will be the average of the fractional $\ell$-lists $Q_\pi$ thus constructed over all permutations $\pi$.

We show, for an appropriate choice of the retention probability $\gamma$, that the transductive menu-conditional loss of this fractional list predictor is competitive with the menu-conditional loss of any fixed hypothesis $h \in \mcH$, or in other words, $Q$ has small agnostic transductive error. In particular, we show that the excess error scales roughly as $\sqrt{\log(s(V_0))/n}$, where $V_0$ is the initial version space of the class $\mcH$ conditioned on the labels belonging to the menu $\nu$. Since the $\ell$-Natarajan dimension of the class is bounded by $\dnat$, and the size of the menu is at most $p=\widetilde{O}(\ell \dds)$, we can bound $s(V_0)$ roughly as $n^{\dnat} p^\ell$; this gives us the final desired scaling for sample complexity. 

Once we have this agnostic transductive guarantee for the fractional list predictor, we can directly appeal to the result of \cite{dughmi25transductive}, which shows how any agnostic transductive learner may be converted to an agnostic learner in the PAC setting, with essentially no blowup in sample complexity. Summarily, this line of reasoning yields a fractional list predictor $q$ that has small agnostic menu-conditional error.

It remains to round this fractional list predictor to a precise set of at most $\ell$ labels, such that the rounded predictor has error not much worse than the fractional predictor. We obtain this rounded list by plotting the cumulative probabilities over the menu of labels on the real line, sampling a $u \in (0,1)$ uniformly at random, and then outputting the at most $\ell$ labels spanned by $u, u+1,\dots, u+\ell-1$. We argue that in expectation, the list predictor thus obtained has menu-conditional error no worse than the fractional list predictor. Finally, we determine the actual value of $u$ to use by drawing a small additional sample to validate the loss. This finally gives us an exact $\ell$-list predictor that has small menu-conditional agnostic loss, which, thanks to Steps (1) and (2), suffices for the overall agnostic learning guarantee that we desire.

\paragraph{Lower Bound.} The lower bound constructions require carefully extending the standard KL-divergence-based machinery for establishing agnostic sample complexity lower bounds in PAC learning to the list learning setting. Recall that the benchmark in agnostic list learning is still the best \textit{single-label assigning hypothesis} from the class. Namely, even if the problem becomes easier when we transition from standard multiclass to list learning, the benchmark for agnostic learning remains the same. This makes proving lower bounds for agnostic list learning more challenging.

We show the necessity of the $\Omega(\log(1/\delta)/\eps^2)$ and $\Omega(\ell \dnat/\eps^2)$ terms in the sample complexity separately. Here, we sketch the construction for the former term; the construction for the latter term follows a similar template, albeit over a slightly more involved, ``tensorized'' class. 

For the $\Omega(\log(1/\delta)/\eps^2)$ term, we consider a class with $\mcX = \N$, and $\mcY = \mcY_0 \sqcup \mcY_1$, where $\mcY_0$, $\mcY_1$ are disjoint and $|\mcY_0| = |\mcY_1| = \ell$. The class comprises of all the functions mapping $\mcX \to \mcY_0$, together with all the functions mapping $\mcX \to \mcY_1$; it can be readily verified that the $\ell$-DS dimension of this class is 1. We then consider two distributions $p_0$ and $p_1$ over $\mcY$, where $p_0$ prefers labels in $\mcY_0$ over $\mcY_1$, and vice versa for $p_1$. More precisely, for $\theta \in \{0,1\}$, we have
    \begin{align*}
        p_\theta(y) = \begin{cases}
            \frac{1+4\eps}{2\ell} & \text{if $y \in \mcY_\theta$}, \\
            \frac{1-4\eps}{2\ell} & \text{otherwise}.
        \end{cases}
    \end{align*}
We then draw $\theta$ uniformly at random from $\{0,1\}$, and for a suitably large integer $M$, draw labels for the first $M$ natural numbers $F(1),\dots,F(M)$ i.i.d.\ from $p_\theta$. We then set the data distribution $\mcD_F$ to be uniform over $\{(x,F(x)): x \in [M]\}$. Then, if we further draw an i.i.d.\ sample $S$ of size $m$ from $\mcD_F$, we can show that $\KL(P^0(S) || P^1(S)) \le O(m\eps^2)$, where $P^\theta(S)$ is the distribution of $S$ conditioned on $\theta$. But this means, by the standard two-point testing lower bound, that the error $\Pr[\widehat{\theta} \neq \theta]$ of any estimator $\widehat{\theta}$ for $\theta$ is $\gtrsim \exp(-m\eps^2)$.

The next step is to define an estimator $\widehat{\theta}$ for $\theta$ from the output of a given list learner $\mcA(S)=\mu_S$, and then lower-bound the excess error of $\mu_S$  by the error of the estimator. Our benchmark hypothesis from the class will be the hypothesis $h_{\theta, F}$ which agrees with all the labels in $\mcY_\theta$. We define $\widehat{\theta}$ as follows: for the points $x$ in $[M]$ that are \textit{unseen} in the sample $S$, we look at the list of labels $\mu_S(x)$ assigned by $\mu_S$, and count the total number of labels that belong to $\mcY_0$ and $\mcY_1$. If the learner has done a good job at learning, then a natural estimator for $\theta$ would be to output $0$ if there are more labels from $\mcY_0$ on the unseen points, and $1$ if there are more labels from $\mcY_1$ --- this is because labels from $\mcY_\theta$ are more likely under $p_\theta$. This is precisely the estimator $\widehat{\theta}$ that we consider. Then, if we condition on $\widehat{\theta} \neq \theta$ --- which, by our reasoning above, happens with probability $\gtrsim \exp(-m\eps^2)$ --- we can derive that $\mu_S$ outputs a lot of labels from the wrong class $\mcY_{1-\theta}$ in expectation. In particular, since the benchmark hypothesis $h_{\theta, F}$ agrees with all the labels from $\mcY_\theta$, we can show that the excess error of $\mu_S$ over $h_{\theta, F}$ is larger than $\eps$ with probability $\gtrsim \exp(-m\eps^2)$. This latter probability is at least $\delta$ if $m \lesssim \log(1/\delta)/\eps^2$, which is the desired lower bound.

The construction for the $\Omega(\ell \dnat/\eps^2)$ term uses $d=\dnat$ blocks of $\N$ as the domain $\mcX$ (so that a random sample of $m$ points contains $m/d$ points from a given block in expectation). Furthermore, whereas in the instance above, a single bit $\theta$ determined whether all $\ell$ labels within $\mcY_0$ or $\mcY_1$ are more likely, here, we will have a separate vector $\theta^{(b)} \in \{-1,1\}^\ell$ for every block $b$, and $\theta^{(b)}_j$ determines whether the $j^{\text{th}}$ label within $\mcY_0$ or $\mcY_1$ is more likely for this block. For a similarly constructed random experiment as the one above, we can then argue that the KL divergence between the distribution of an $m$-sized sample under two assignments to $(\theta^{(b)})_{b \in [d]}$ that differ \textit{only} in a single bit is $\lesssim O(m\eps^2/d \ell)$. The rest of the argument then follows along the same lines as above. 

%% file: preliminaries.tex
\section{Preliminaries}
\label{sec:preliminaries}

We will denote the unlabeled data domain by $\mcX$ and the label space by $\mcY$. The label space is allowed to be infinite. %
A hypothesis class $\mcH$ is a subset of $\mcY^\mcX$. The error of any hypothesis $h$ with respect to a data distribution $\mcD$ over $\mcX \times \mcY$ is defined as $\err_\mcD(h)=\Pr_{(x,y) \sim \mcD}[h(x) \neq y]$; for a list hypothesis $\mu$, the error is $\err_\mcD(\mu)=\Pr_{(x,y) \sim \mcD}[y \notin \mu(x)]$. A distribution $\mcD$ is \textit{realizable} by $\mcH$ if $\inf_{h \in \mcH}\err_\mcD(h)=0$.

\begin{definition}[PAC learning \cite{valiant1984theory}]
    \label{def:pac-learnability}
    Let $\mcH \subseteq \mcY^{\mcX}$ be a hypothesis class. We say that $\mcH$ is (list) PAC learnable by a learning algorithm $\mcA$ with sample complexity $m_{\mcA, \mcH}: (0,1) \times (0,1)\to\N$ if for every $\eps, \delta \in (0,1)^2$, every distribution $\mcD$ over $\mcX \times \mcY$ realizable by $\mcH$, for $m \ge m_{\mcA, \mcH}(\eps, \delta)$, $\Pr_{S \sim \mcD^m}[\err_{\mcD}(\mcA(S)) \ge \eps]\le\delta$, where $\mcA(S)$ is the (list) hypothesis output by the learning algorithm on input $S$.
\end{definition}
The definition of agnostic PAC learning allows for arbitrary distributions over $\mcX \times \mcY$ not necessarily realizable by $\mcH$, and the guarantee required of the algorithm is instead $$\Pr_{S \sim \mcD^m}[\err_{\mcD}(\mcA(S)) \ge \inf_{h \in \mcH} \err_\mcD(h) + \eps] \le  \delta,$$ for $m \ge m_{\mcA, \mcH}(\eps, \delta)$.

We now define several combinatorial quantities relevant to multiclass and list learning. These quantities will be parameterized in terms of a list size $\ell \ge 1$, where $\ell$ appears in the superscript of the quantity (e.g., $\dds,\, \mu^\ell_{\mcH}(n)$). When $\ell=1$, we sometimes drop the superscript for convenience, and simply refer to the corresponding quantity without the superscript (e.g., $d_{\mathrm{DS}}, \, \mu_{\mcH}(n)$). %
For a sequence $S \in \mcX^n$, we denote the restriction of $\mcH$ on $S$ as $\mcH|_S$. %

\begin{definition}[$\ell$-DS dimension \cite{daniely2014optimal,charikar2023characterization}]
    \label{def:ell-ds-dim}
    Let $\mcH \subseteq \mcY^\mcX$ be a hypothesis class, and let $S \in \mcX^d$ be a sequence. Let us think of the members of $\mcH|_S$ as vectors in $\mcY^d$. For any $i \in [d]$, we say that $f,g \in \mcH|_S$ are $i$-neighbors if $f_i \neq g_i$ and $f_j = g_j, \; \forall j \neq i$.
    For any $\ell \ge 1$, we say that $\mcH$ $\ell$-DS shatters $S$ if there exists a non-empty $\mcF|_S \subseteq \mcH|_S, |\mcF|_S|<\infty$ such that $\forall f \in \mcF|_S ,\; \forall i \in [d],$ $f$ has at least $\ell$ $i$-neighbors in $\mcF|_S$. The $\ell$-DS dimension of $\mcH$, denoted as $\dds=\dds(\mcH)$, is the largest integer $d$ such that $\mcH$ $\ell$-DS shatters some sequence $S \in \mcX^d$.
\end{definition}

A different dimension that will also be relevant to our discussion is the $\ell$-Natarajan dimension.

\begin{definition}[$\ell$-Natarajan dimension \cite{natarajan1988two,natarajan1989learning,charikar2023characterization}]
    \label{def:ell-natarajan-dim}
    A hypothesis class $\mcH \subseteq \mcY^{\mcX}$ $\ell$-Natarajan shatters a sequence $S \in \mcX^d$ if there exist $(\ell+1)$-sized lists $y_i \in \{Y \subseteq \mcY: |Y|=\ell+1\}$, $i=1,\dots,d$, such that $\prod_{i=1}^{d}y_i \subseteq \mcH|_{S}$. The $\ell$-Natarajan dimension of $\mcH$, denoted as $\dnat=\dnat(\mcH)$, is the largest integer $d$ such that $\mcH$ $\ell$-Natarajan shatters some sequence $S \in \mcX^d$.
\end{definition}

Observe that $\dnat \le \dds$ for every $\ell \ge 1$. %
We now define a key combinatorial object for a hypothesis class: the one-inclusion graph.

\begin{definition}[One-inclusion graph \cite{haussler1994predicting, rubinstein2006shifting}]
    \label{def:oig}
    The one-inclusion graph of $\mcH \subseteq \mcY^n$ is a hypergraph $\mcG(\mcH)=(V,E)$ that is defined as follows. The vertex set is $V=\mcH$. For each $i \in [n]$ and $f:[n]\setminus\{i\} \to \mcY$, let $e_{i,f}$ be the set of all $h \in \mcH$ that agree with $f$ on $[n]\setminus \{i\}$. The edge set is the multiset:
    \begin{equation}
        \label{eqn:oig-edge-set}
        E = \{e_{i,f}: i \in [n], f:[n] \setminus \{i\} \to \mcY, e_{i,f}\neq \emptyset \}.
    \end{equation}
    We say that the edge $e_{i,f}\in E$ is in the direction $i$, and is adjacent to/contains the hypothesis/vertex $h$ if $h \in e_{i,f}$. Every vertex $h\in V$ is adjacent to exactly $n$ edges. The size of the edge $e_{i,f}$ is the size of the set $|e_{i,f}|$. %
\end{definition}

We proceed to define the other key quantity in \Cref{thm:main-upper-bound}, namely the maximum density function $\muh(n)$ of the one-inclusion graph.

\begin{definition}[$\ell$-density]
    \label{def:ell-density}
    Let $\mcG(\mcH)=(V,E)$ be the one-inclusion graph of $\mcH \subseteq \mcY^n$, where $0 <|\mcH| < \infty$. For any integer $\ell \ge 1$, the $\ell$-density of $\mcH$ is
    \begin{align*}
        \dens^\ell(\mcH) = \frac{1}{|V|}\sum_{e \in E}(|e|-\ell)_+, 
    \end{align*}
    where $(x)_+ = \max(x,0)$.
\end{definition}
Observe that for every finite $\mcH \subseteq \mcY^n$, $\dens^\ell(\mcH) \le n$.

\begin{definition}[Maximum $\ell$-density function]
    \label{def:maximum-ell-density-function}
    Let $\mcH \subseteq \mcY^\mcX$ be a hypothesis class. For any integer $\ell \ge 1$, the maximum $\ell$-density function of $\mcH$ for sample size $n$ 
    \begin{align*}
        \muh(n) = \sup_{S \in \mcX^n} \sup_{\substack{\mcF|_S \subseteq \mcH|_S, \\ 0 <|\mcF|_S| < \infty}} \dens^\ell(\mcF|_S).
    \end{align*}
\end{definition}

\begin{remark}[Differing Names, Differing Definitions]
    \label{remark:differing-definitions}
    There has been some unfortunate inconsistency in the literature in how the quantities above are named/defined. In their original paper, \cite{daniely2014optimal} defined $\mu_\mcH(n)$ differently. Their definition was as follows:
    \begin{align}
        \label{eqn:other-definition-muh}
        \mu'_\mcH(n) = \sup_{S \in \mcX^n} \sup_{\substack{\mcF|_S \subseteq \mcH|_S, \\ 0 <|\mcF|_S| < \infty}} \frac{1}{|\mcF|_S|}\sum_{e \in E: |e| > 1}|e|.
    \end{align}
    Indeed, in the original statement of their conjecture, \cite{daniely2014optimal} asked whether $\mu'_\mcH(n) \le c \cdot d_{\mathrm{DS}}$. However, we can verify that $\frac{\mu'_\mcH(n)}{2} \le \mu_\mcH(n) \le \mu'_\mcH(n)$, and so the discrepancy is only upto a factor of 2 (and $\ell+1$ more generally).

    It is also worth mentioning that while \cite{haussler1994predicting} originally also referred to $\dens(\mcH)$ defined above in \Cref{def:ell-density} as ``density'', this quantity in the case of $\ell > 1$ has been instead referred to as the ``shifting average $\ell$-degree'' in \cite{charikar2023characterization,hanneke2026optimal}.
\end{remark}

%% file: main-upper-bound.tex
\section{Main Structural Result}
\label{sec:main-upper-bound}

In this section, we restate \Cref{thm:main-upper-bound}, and prove it for the general case of $\ell \ge 1$. For convenience, we recall the definition of the special set of monomials $\mcM^\ell_s(W)$, which was defined with respect to a hypothesis class $W \subseteq [k]^n$ where $k < \infty$, and was shown by \cite{hanneke2026optimal} to span $\mcV_{W}$, the vector space of functions mapping $W$ to $\R$, whenever $s \ge \dds(W)$.
\begin{align*}
    \mcM^\ell_s(W) :=&\left\{f : W \to \R ~\Big\vert~ f(w)=w_1^{\alpha_1} w_2^{\alpha_2}\dots w_n^{\alpha_n} \; \forall w \in W, \right. \nonumber \\
    &\hspace{2.4cm}\left. \alpha_i \in \{0,1,\dots,k-1\}
    \; \forall i \in [n], \right. \nonumber \\
    &\hspace{2.4cm}\left. \left|\{i \in [n]:\alpha_i \ge \ell\}\right| \le s \right\}.
\end{align*}

\TheoremMain*
\begin{proof}
    Fix $\ell \ge 1$. When $n \le \dds$, the inequality is immediate, since $\muh(n) \le n \le \dds$. So, assume that $n > \dds$. For notational ease, let $s = \dds$.

    Fix any $S \in \mcX^n$, finite non-empty $\mcF|_S \subseteq \mcH|_S$, and let $W := \mcF|_S$ denote this restriction for notational ease. Since $W$ is finite, we can without loss of generality (e.g., by coordinatewise relabeling) assume that $W \subseteq [k]^n$, for some finite $k$. Note that the $\ell$-DS dimension of $W$ is at most $s$, since the dimension cannot increase upon projection. Consider the set of monomials $\mcM^\ell_s(W)$ defined in \eqref{eqn:monomial-set-def}. By \Cref{lemma:spanning}, we know that $\mcM^\ell_s(W)$ spans $\mcV_W$. So, let $\mcB \subseteq \mcM^\ell_s(W)$ be a basis for $\mcV_W$. Note again that $\mcB$ comprises of monomials satisfying the condition in \eqref{eqn:monomial-set-def} that defines the set $\mcM^\ell_s(W)$; furthermore,
    \begin{align*}
        |\mcB| = \Dim(\mcV_W) = |W|.
    \end{align*}

    Now, for every $i \in [n]$, we will define a subspace $U_i \subseteq \mcV_W$; we first set up some notation. For any $w \in W$, let $w_{-i} \in [k]^{n-1}$ denote its behavior on $[n] \setminus \{i\}$, and let $T_i := \{w_{-i}: w \in W\} \subseteq [k]^{n-1} $ comprise of all such distinct behaviors. Additionally, let $\mcP_{< \ell}$ comprise of all univariate real-valued polynomials of degree at most $\ell-1$ mapping $[k]$ to $\R$. Then, the subspace $U_i$ comprises of all functions $f: W \to \R$, such that if we restrict $f$ to $w \in W$ that share the same behavior $a$ on $[n] \setminus \{i\}$, then $f$ is simply a univariate polynomial $p_a$ of degree at most $\ell-1$ in $w_{i}$. Formally,
    \begin{align}
        \label{eqn:def-U-i}
        U_i := \{f:W \to \R ~\mid~ \forall a \in T_i \, \exists p_a \in \mcP_{< \ell}\, \forall w \in W \,: \, (w_{-i} = a \implies f(w)=p_a(w_i))\}.
    \end{align}
        We can verify that $U_i$ is indeed a subspace of $\mcV_W$. We now claim that $\dim(U_i)=\sum_{a \in T_i} \min(\ell, |e_{i, a}|)$. This is more readily seen in the case of $\ell=1$; in this case, any $f \in U_i$ is required to be constant on all $w$ that share the same behavior $a$ on $[n] \setminus \{i\}$. Then, $\dim(U_i)=|T_i|$, corresponding to exactly one degree of freedom for every behavior on $[n] \setminus \{i\}$.
        
        More generally, observe that the edge sets $e_{i,a}$ for $a \in T_i$ partition $W$, and for $a,a' \in T_i$, the constraint on $f \in U_i$ at points $w$ satisfying $w_{-i}=a$ is independent of the constraint at points $w$ satisfying $w_{-i}=a'$. Going back to our viewpoint of $U_i$ as vectors in $\R^{|W|}$, we can thus consider the restrictions of $f \in U_i$ on the members of an edge $e_{i,a}$, which constitute a subset $W_a \subseteq \R^{|e_{i,a}|}$, and obtain that $\dim(U_i)=\sum_{a \in T_i} \dim(W_a)$.

        It remains to argue that $\dim(W_a)=\min(\ell, |e_{i,a}|)$. Let $|e_{i,a}|=t$, and suppose $e_{i,a}=\{h^{(1)},\dots,h^{(t)}\}$. Denote $z_j = h^{(j)}_{i}$. Observe that by definition of the edge, $z_j \neq z_{j'}$. Furthermore, by the constraint in \eqref{eqn:def-U-i}, $U_i$ contains precisely those $f$ for which $f(h^{(j)})=p_a(z_j)$ for all $j \in [t]$, for some $p_a(z)=c_0+c_1z +\dots+c_{\ell-1}z^{\ell-1}$. That is, the members of $W_a$ are of the form
        \begin{align*}
            \begin{bmatrix}
                f(h^{(1)}) \\
                \vdots \\
                f(h^{(t)})
            \end{bmatrix} = 
            c_0 \begin{bmatrix}
                1 \\
                \vdots \\
                1
            \end{bmatrix} +
            c_1 \begin{bmatrix}
                z_1 \\
                \vdots \\
                z_t
            \end{bmatrix} +
            \dots +
            c_{\ell-1} \begin{bmatrix}
                z_1^{\ell-1} \\
                \vdots \\
                z_t^{\ell-1}
            \end{bmatrix},
        \end{align*}
        or equivalently, $W_a$ is the column span of the $t \times \ell$ Vandermonde matrix
        \begin{align*}
            M := \begin{bmatrix}
                1 & z_1 & z_1^2 & \dots & z_1^{\ell-1} \\
                1 & z_2 & z_2^2 & \dots & z_2^{\ell-1} \\
                \vdots & \vdots & \vdots &  & \vdots \\
                1 & z_t & z_t^2 & \dots & z_t^{\ell-1} \\
            \end{bmatrix}.
        \end{align*}
        We then claim that $\dim(W_a)=\rank(M)=\min(\ell,t)$. To see this, suppose first that $t \le \ell$. In this case, the square submatrix of $M$ comprising of the first $t$ rows and columns is the $t \times t$ Vandermonde matrix. Its determinant is $\prod_{1 \le i < j \le t}(z_j-z_i) \neq 0$, since for every $i \neq j$, $z_j \neq z_i$ as argued above. Thus, in this case, $\rank(M)=t$. On the other hand, suppose that $t > \ell$. In this case, the square submatrix of $M$ comprising of the first $\ell$ rows and columns is the $\ell \times \ell$ Vandermonde matrix. Again, its determinant is $\prod_{1 \le i < j \le \ell}(z_j-z_i) \neq 0$. So, in this case, $\rank(M)=\ell$. We conclude that $\dim(W_a)=\rank(M)=\min(\ell, t)$ as claimed. Hence,
        \begin{align*}
            \dim(U_i) = \sum_{a \in T_i}\dim(W_a) = \sum_{a \in T_i}\min(\ell, |e_{i,a}|).
        \end{align*}

        Next, for any $i \in [n]$, let $\mcB^{< \ell}_{i} \subseteq \mcB$ be the monomials in the basis $\mcB$ that have degree smaller than $\ell$ at coordinate $i$. That is,
        \begin{align*}
            \mcB^{< \ell}_i := \{f \in \mcB ~\mid~ f(w)=w_1^{\alpha_1}\dots w_n^{\alpha_n} \; \forall w \in W, \; \alpha_i < \ell\}.
        \end{align*}
        Further, let $\mcB^{\ge \ell}_{i} = \mcB \setminus \mcB^{< \ell}_i$. Observe crucially, that every monomial in $\mcB^{<\ell}_{i}$ also belongs to $U_i$. This is because for any fixed behavior $a$ on $[n] \setminus \{i\}$, if we consider the evaluation of the monomial on any member $h^{(j)}$ of $e_{i,a}$, this equals some fixed constant times $(h^{(j)}_i)^{\alpha_i}$, and $\alpha_i < \ell$ by definition of $\mcB^{< \ell}_i$. On the other hand, by virtue of $\mcB$ being a basis, it also holds that all monomials in $\mcB^{< \ell}_i$ are linearly independent. Thus,
        \begin{alignat*}{2}
            &|\mcB^{< \ell}_i| \le \dim(U_i) = \sum_{a \in T_i}\min(\ell, |e_{i,a}|) &&\\
            \implies \quad & |\mcB^{\ge \ell}_i| = |\mcB| - |\mcB^{< \ell}_i| \ge |\mcB| - \sum_{a \in T_i}\min(\ell, |e_{i,a}|) &&= |W| - \sum_{a \in T_i}\min(\ell, |e_{i,a}|) \\
            & &&= \sum_{a \in T_i} |e_{i,a}| - \sum_{a \in T_i}\min(\ell, |e_{i,a}|) \tag{since the edge sets $e_{i,a}$ for $a \in T_i$ partition $W$} \\
            & &&= \sum_{a \in T_i}(|e_{i,a}|-\ell)_{+}.
        \end{alignat*}
        Summing over all $i \in [n]$, we get
        \begin{align*}
            \sum_{i=1}^{n} \sum_{a \in T_i}(|e_{i,a}|-\ell)_{+} \le \sum_{i=1}^n |\mcB^{\ge \ell}_i| \implies
            |W| \cdot \dens^\ell(W) \le \sum_{i=1}^n |\mcB^{\ge \ell}_i|. \tag{since the left term sums $(|e|-\ell)_{+}$ over all edges $e$ in $\mcG(W)$}
        \end{align*}
        But now, observe that the sum $\sum_{i=1}^n |\mcB^{\ge \ell}_i|$ counts, over all the basis monomials in $\mcB$, the number of coordinates $i$ that have degree at least $\ell$. Since $\mcB \subseteq \mcM^\ell_s(W)$, this number is at most $s$, by definition of $\mcM^\ell_s(W)$ (see \eqref{eqn:monomial-set-def}). Thus, we get that
        \begin{align*}
            |W| \cdot \dens^\ell(W) \le \sum_{i=1}^n |\mcB^{\ge \ell}_i| \le \sum_{f \in \mcB} s = s|\mcB| = s|W|.
        \end{align*}
        Dividing by $|W|$ gives $\dens^\ell(W) \le s = \dds$. Thus, we have shown that for every $S \in \mcX^n$ and finite non-empty $\mcF|_S \subseteq \mcH|_S$, denoting $W = \mcF|_S$, it holds that $\dens^\ell(W) \le \dds$. It follows that 
        \begin{align*}
            \muh(n) = \sup_{S \in \mcX^n} \sup_{\substack{\mcF|_S \subseteq \mcH|_S, \\ 0 <|\mcF|_S| < \infty}} \dens^\ell(\mcF|_S) \le \dds.
        \end{align*}
        Finally, since $\dds$ is always an integer, it must hold that $\lceil \muh(n)\rceil \le \dds$, completing the proof.
\end{proof}

\begin{remark}[Optimality of the Constant]
    \label{remark:optimality-of-constant}
    We note that the upper bound in \Cref{thm:main-upper-bound} is tight including the constant. To see this, consider the hypothesis class $\mcH = [k]^s \times [\ell]^{m-s}$ as $k$ gets large. We have that $\dds(\mcH)=s$, whereas \begin{align*}
        \lceil\mu^\ell_\mcH(m)\rceil \ge \dens^\ell(\mcH) = \frac{1}{k^s\ell^{m-s}}\left(s \cdot k^{s-1}\cdot \ell^{m-s} \cdot (k-\ell)\right) = s\left( 1-\frac{\ell}{k}\right) \xrightarrow[k \to \infty]{} s = \dds(\mcH).
    \end{align*}
    For $\ell=1$, even if we consider the definition $\mu'_{\mcH}(n)$ originally considered by \cite{daniely2014optimal} (and defined in \eqref{eqn:other-definition-muh}), \cite{daniely2014optimal} had already shown the lower bound $d_\mathrm{DS} \le \mu'_{\mcH}(n)$. As discussed in \Cref{remark:differing-definitions}, $\frac{\mu'_\mcH(n)}{2} \le \mu_\mcH(n)$, meaning that %
    the upper bound implied by \Cref{thm:main-upper-bound} for this definition of $\mu'_\mcH(n)$ is loose by at most a factor of 2.
\end{remark}

%% file: list-learning-results.tex
\section{Learning Results}
\label{sec:list-learning}

In this section, we prove all our results for learning with a given list size $\ell \ge 1$. Note that the special case of multiclass learning is captured by setting $\ell=1$.

\subsection{Realizable Case}
\label{sec:realizable-list-learning}

\subsubsection{Deterministic List Learner}
\label{sec:near-optimal-deterministic-list-learner}

We first show how the machinery of \cite{aden2023optimal} can be directly combined with \Cref{thm:main-upper-bound} to give a deterministic list learner that has sample complexity $O\left(\frac{\ell(\dds+\log(1/\delta))}{\eps}\right)$.

\begin{corollary}[Near-optimal Sample Complexity for List Learning]
    \label{corollary:list-near-optimal-sample-complexity}
    Let $\mcH \subseteq \mcY^\mcX$ be a hypothesis class having $\ell$-DS dimension $\dds$, for any $\ell \ge 1$. There exists a deterministic list learning algorithm $\mcA$, such that for any distribution $\mcD$ over $\mcX \times \mcY$ realizable by $\mcH$, and any $\eps, \delta \in (0,1)^2$, with probability at least $1-\delta$ over a sample $S \sim \mcD^m$, where
    \begin{align*}
        m \ge  4.82(\ell+1) \left(\frac{\dds+\log(2/\delta)}{\eps}\right),
    \end{align*}
    it holds that $\err_\mcD(\widehat{\mu}_S) \le \eps$, where $\widehat{\mu}_S=\mcA(S)$.
\end{corollary}

\begin{proof}[Proof of \Cref{corollary:list-near-optimal-sample-complexity}]
    Consider the one-inclusion graph list predictor $\mcA$ given in Algorithm 1 in \cite{charikar2023characterization}. At a high level, given a training dataset $S=\{(x_1,y_1),\dots,(x_n,y_n)\}$, in order to make a prediction on any test point $x$, the algorithm constructs the one-inclusion graph $\mcG(\mcH|_{(x_1,\dots,x_n,x)})$ of the class $\mcH$ projected onto the unlabeled data $(x_1,\dots,x_n,x)$. It then maps/\textit{orients} every edge $e$ in the graph to a subset of the hypotheses adjacent to/contained in $e$, where this subset has size at most $\ell$. For any such ``$\ell$-list'' orientation $\sigma^\ell$, one can define the $\ell$-outdegree of a vertex $v$ in the graph (denoted $\outdeg^\ell(v;\sigma^\ell)$), which is the count of edges adjacent to $v$ that have been oriented to a subset that does \textit{not} include $v$. The one-inclusion graph list predictor $\mcA$ then specifically constructs an orientation $\sigma^\ell$ for which the maximum outdegree of any vertex according to that orientation is minimized. To predict a list of labels for the test point $x$, the algorithm then looks at the edge corresponding to the labeled training data $(x_1,y_1),\dots,(x_n,y_n)$. The orientation $\sigma^\ell$ would have oriented this edge to some set of at most $\ell$ hypotheses. The algorithm outputs the list of labels that these hypotheses assign to $x$.

    Given any training dataset $S=\{(x_1,y_1),\dots,(x_n,y_n)\}$ realizable by $\mcH$, we will be concerned with the \textit{leave-one-out error} of $\mcA$, defined as 
    \begin{align*}
        M_n := \sum_{i=1}^n \Ind\left[y_i \notin \widehat{\mu}^\ell_{S_{-i}}(x_i) \right],
    \end{align*}
    where $S_{-i} = S \setminus \{(x_i,y_i)\}$, and $\widehat{\mu}^\ell_{S_{-i}}=\mcA(S_{-i})$ is the output of the one-inclusion graph list predictor trained on $S_{-i}$. Because the algorithm constructs the same one-inclusion graph irrespective of the point $x_i$ that is held out as the test point, it follows that
    \begin{align*}
        \sum_{i=1}^n \Ind\left[y_i \notin \widehat{\mu}^\ell_{S_{-i}}(x_i)  \right] = \outdeg^\ell(y;\sigma^\ell),
    \end{align*}
    where $y=(y_1,\dots,y_n)$ is the vertex corresponding to the ground-truth labels.

    Now, 
    \Cref{lemma:compactness} proved in \Cref{sec:infinite-label-space} shows that for any class $W \subseteq \mcY^n$
    there exists an $\ell$-list orientation $\sigma^\ell$ of the edges of $\mcG(W)$ such that
    \begin{align}
        \max_{v \in \mcG(W)} \outdeg^\ell(v;\sigma^\ell) %
        \le \left\lceil\mu^\ell_{W}(n)\right\rceil. \label{eqn:outdegree-bound}
    \end{align}
    But by \Cref{thm:main-upper-bound}, the right-hand side above is upper-bounded by $\dds(W)$. Plugging this into the leave-one-out error bound above, we get that
    \begin{align}
        \label{eqn:ell-outdegree-2}
        M_n = \sum_{i=1}^n \Ind\left[y_i \notin \widehat{\mu}^\ell_{S_{-i}}(x_i) \right] \le \dds(\mcH).
    \end{align}
    We now observe that the one-inclusion graph list predictor $\mcA$ satisfies both the assumptions --- symmetry and bounded leave-one-out error --- required of Theorem 2.1 in \cite{aden2023optimal}, with the loss function $L:\mcY^{\le \ell} \times \mcY \to [0,1]$ interpreted as $L(\mu, y)=\Ind[y \notin \mu]$. We can therefore apply that theorem, and obtain that
    \begin{align}
        \label{eqn:martingale-bound}
        \frac{4}{3n}\sum_{t=n/4}^{n-1} \err_\mcD\left( \widehat{\mu}^\ell_{S_{\le t}} \right) \le  4.82 \left( \frac{M_n+\log(2/\delta)}{n}\right) \le  4.82 \left( \frac{\dds(\mcH)+\log(2/\delta)}{n}\right),
    \end{align}
    with probability at least $1-\delta$ over the draw of $S=\{(x_1,y_1),\dots,(x_n,y_n)\}$ from $\mcD^n$, and where $S_{\le t}=\{(x_1,y_1),\dots,(x_t,y_t)\}$. But then, consider the list predictor defined by the $\topk$ vote
    \begin{align*}
        \widehat{\mu}^\ell_S(x) = \topk\left(\widehat{\mu}^\ell_{S_{\le n/4}}(x),\dots,\widehat{\mu}^\ell_{S_{\le n-1}}(x)\right),
    \end{align*}
    where $\topk$ returns the list of the first $\ell$ labels sorted in descending order according to the count of lists in $\widehat{\mu}^\ell_{S_{\le n/4}}(x),\dots,\widehat{\mu}^\ell_{S_{\le n-1}}(x)$ that they occur in. Crucially, whenever the $\topk$ vote $\widehat{\mu}^\ell_S(x)$ \textit{does not} contain a label $y$, it must necessarily hold that at least $\frac{3n/4}{\ell+1}$ out of the $3n/4$ lists $\widehat{\mu}^\ell_{S_{\le n/4}}(x),\dots,\widehat{\mu}^\ell_{S_{\le n-1}}(x)$ do not contain $y$; otherwise, the $\topk$ vote would have picked $y$. In conclusion, we get that
    \begin{align*}
        \err_\mcD(\widehat{\mu}^\ell_S) &= \E_{(x,y) \sim \mcD} \left[\Ind\left[y \notin \widehat{\mu}^\ell_S(x)\right]\right] \\
        &\le \E_{(x,y) \sim \mcD} \left[(\ell+1)\left(\frac{4}{3n}\sum_{t=n/4}^{n-1}\Ind[y \notin \widehat{\mu}^\ell_{S_{\le t}}(x)]\right)\right] \\
        &= (\ell+1) \cdot \frac{4}{3n}\sum_{t=n/4}^{n-1} \err_\mcD\left( \widehat{\mu}^\ell_{S_{\le t}} \right) \\
        &\le 4.82(\ell+1) \left( \frac{\dds(\mcH)+\log(2/\delta)}{n}\right) \tag{from \eqref{eqn:martingale-bound}}.
    \end{align*}
\end{proof}

\subsubsection{Randomized List Learner}
\label{sec:optimal-randomized-list-learner}

Building up on the deterministic list learner from the previous section, and inspired from the confidence amplification technique of \cite{haussler1994predicting}, we now suitably incorporate randomness and validation on a holdout set to obtain a randomized list learner that has optimal sample complexity. 

\RealizableList*

\begin{proof}%
    Choose a constant $C$ to be appropriately large enough, so that, by Theorem 2.1 of \cite{aden2023optimal}, upon drawing $S = \{(x_1,y_1),\dots,(x_n,y_n)\}$ i.i.d.\ from $\mcD$, where
    \begin{align*}
        n = C \cdot \left(\frac{\dds(\mcH)+\log(1/\delta)}{\eps}\right),
    \end{align*}
    it holds with probability $1-\delta/3$ that
    \begin{align*}
        \frac{4}{3n}\sum_{t=n/4}^{n-1} \err_\mcD\left( \widehat{\mu}^\ell_{S_{\le t}} \right) \le 4.82 \left( \frac{\dds(\mcH)+\log(6/\delta)}{n}\right) \le \frac{\eps}{16},
    \end{align*}
    where $S_{\le t}=\{(x_1,y_1),\dots,(x_t,y_t)\}$. This set will constitute the learner's ``training set''. 
    
    Let us condition on the good event above. On this good event, since the average error of the predictors trained on prefixes is at most $\eps/16$, the fraction of predictors that have error larger than $\eps/4$ is at most $1/4$. In particular, this means that if we draw a uniformly random $t \in \{n/4,\dots,n-1\}$, it holds that
    \begin{align}
        \label{eqn:large-error-small-fraction}
        \Pr_t\left[\err_\mcD\left( \widehat{\mu}^\ell_{S_{\le t}} \right) > \frac{\eps}{4}\right] \le \frac{1}{4}.
    \end{align}
    Our strategy will now be to randomly sample a sufficient number of prefixes and train predictors on them, so that with good probability, one of the predictors will have error at most $\eps/4$.

    Concretely, set $R = O(\log(1/\delta))$, and obtain $t_1,\dots,t_R$ each drawn uniformly at random from $\{n/4,\dots,n-1\}$. From \eqref{eqn:large-error-small-fraction}, it holds that with probability at least $1-\delta/3$, there exists $t_{r^\star}$ such that
    \begin{align*}
        \err_\mcD\left( \widehat{\mu}^\ell_{S_{\le t_{r^\star}}} \right) \le \frac{\eps}{4}.
    \end{align*}
    
    Now condition on this good event as well. We want to next be able to identify the low-error predictor from amongst the other (potentially high-error) predictors. For this, we draw a separate independent ``validation set'' $V \sim \mcD^v$ in order to empirically benchmark the error of each predictor.

    Namely, let $\widehat{\err}_V\left( \widehat{\mu}^\ell_{S_{\le t_{r}}} \right)$ be the average empirical validation error of $\widehat{\mu}^\ell_{S_{\le t_{r}}}$ on the validation set $V$. For any $r$ satisfying that $\err_\mcD\left( \widehat{\mu}^\ell_{S_{\le t_{r}}} \right) > \eps$, we have by a Chernoff bound that
    \begin{align*}
        \Pr_V \left[ \widehat{\err}_V\left( \widehat{\mu}^\ell_{S_{\le t_{r}}} \right) \le  \eps/2\right] \le \exp\left(-v\eps/12\right).
    \end{align*}
    Similarly, for the good predictor satisfying that $\err_\mcD\left( \widehat{\mu}^\ell_{S_{\le t_{r^\star}}} \right) \le \eps/4$, the Chernoff bound ensures
    \begin{align*}
        \Pr_V \left[ \widehat{\err}_V\left( \widehat{\mu}^\ell_{S_{\le t_{r^\star}}} \right) \ge  \eps/2\right] \le \exp\left(-v\eps/12\right).
    \end{align*}
    Choosing $v=O(\log(R/\delta)/\eps)$ and applying a union bound over the $R$ predictors ensures that with probability $1-\delta/3$, every bad predictor has validation error greater than $\eps/2$ and the good predictor corresponding to $t_{r^\star}$ has validation error smaller than $\eps/2$. In particular, the learner can output $\widehat{\mu}^\ell_{S_{\le t_{\widehat{r}}}}$, where
    \begin{align*}
        \widehat{r} = \argmin_{r \in [R]} \widehat{\err}_V\left( \widehat{\mu}^\ell_{S_{\le t_{r}}} \right),
    \end{align*}
    which ensures that $\err_\mcD\left( \widehat{\mu}^\ell_{S_{\le t_{\widehat{r}}}} \right) \le \eps$. Combining with a union bound the three failure probabilities --- the average-prefix event, the random prefix-sampling event, and the empirical validation event --- gives that this guarantee holds with probability at least $1-\delta$. The total number of samples required is
    \begin{align*}
        n + v = O\left(\frac{\dds(\mcH)+\log(1/\delta)}{\eps}\right).
    \end{align*}
\end{proof}

\subsubsection{Realizable List Learning Lower Bound}
\label{sec:realizable-list-learning-lower-bound}

In this section, we show that the realizable sample complexity of $\ell$-list learning a class $\mcH$ having $\ell$-DS dimension $\dds$ is at least $\Omega \left(\frac{\dds + \log(1/\delta)}{\eps}\right)$, \textit{provided} that the class satisfies an additional ``label-richness'' property. This strengthens the lower bound established by \cite{hanneke2024improved} by a factor of $\ell$. In particular, while this does not show the lower bound for \textit{every} class having $\ell$-DS dimension $\dds$, it nevertheless does show that one cannot universally improve the sample complexity upper bound of \Cref{thm:list-optimal-sample-complexity}.

We note that the arguments for the improved lower bound are essentially identical to those used by \cite{hanneke2024improved}, but exploit the additional label-richness condition that is assumed. This condition requires $\mcH$ to not only have $\ell$-DS dimension $\dds$, but requires also that $d_{\mathrm{DS}}^{2\ell-1} \ge \dds$.

We also refer the reader to the recent work of \cite{li2026optimisticratesmulticlasspac} which establishes a similar lower bound of $\Omega(\dds/\eps)$; here, we explicitly derive the dependence on $\delta$ as well.

\begin{theorem}[Realizable List Learning Lower Bound]
    \label{thm:realizable-list-learning-lower-bound}
    Let $\mcH$ be a hypothesis class having $\ell$-DS dimension $\dds \ge 2$, for any $\ell \ge 1$. Suppose further that
    $\mcH$ $(2\ell-1)$-DS shatters some sequence of size $\dds$.
    Then, every (possibly randomized) $\ell$-list learner for $\mcH$ has sample complexity %
    \begin{align*}
        m = \Omega \left(\frac{\dds + \log(1/\delta)}{\eps}\right),
    \end{align*}
    for any $(\eps, \delta) \in \left(0,\frac{1}{16}\right)^2$. %
\end{theorem}
\begin{proof}
    Let $\mcA$ be any randomized $\ell$-list learner for $\mcH$, and let us explicitly denote the randomness it uses by $r$. We will show the necessity of the $\dds/\eps$ and $\log(1/\delta)/\eps$ terms separately. This will imply that 
    \begin{align*}
        m = \Omega\left(\max\left\{\frac{\dds}{\eps}, \frac{\log(1/\delta)}{\eps}\right\}\right) = \Omega \left(\frac{\dds + \log(1/\delta)}{\eps}\right)
    \end{align*}
    as desired.
    
    \paragraph{Necessity of $\Omega(\dds/\eps)$.} Let $d=\dds$ for notational convenience; by assumption, there exists a sequence $(x_1,\dots,x_d)$ which is $(2\ell-1)$-DS shattered by $\mcH$. Let $\{h_v\}_{v \in V}$ be the finite subset of $\mcH$ that realizes this shattering (where we think of each $h_v$ as a string of size $d$ over the label space).

    For each $v \in V$, define the realizable distribution $\mcD_v$ as
    \begin{align*}
        \mcD_v((x_1,h_v(x_1))) = 1-8\eps, \qquad \mcD_v((x_j,h_v(x_j))) = \frac{8\eps}{d-1} \,\, \text{for $j \in \{2,\dots,d\}$}.
    \end{align*}
    Now, suppose we draw $v \sim V$ uniformly at random, draw $m$ training samples 
    $$S=\{(s_1, h_v(s_1)),\dots,(s_m, h_v(s_m))\},$$
    and then draw a test sample $(s, h_v(s))$.

    We have that
    \begin{align*}
        \Pr[s=x_1] = 1-8\eps, \qquad \Pr[s=x_j] = \frac{8\eps}{d-1} \,\, \text{for $j \in \{2,\dots,d\}$}.
    \end{align*}

    Let $E_i$ be the event that $x_i$ is not seen in the training sample. Conditioning on any fixed training sample $S$ in $E_i$, and also further conditioning on the randomness $r$ of the learner, the output $\widehat{\mu}_S(x_i)$, where $\widehat{\mu}_S = \mcA(S, r)$, is some fixed list of at most $\ell$ labels. But, since $(x_1,\dots,x_d)$ is $(2\ell-1)$-DS shattered by $\{h_v\}_{v \in V}$, the probability (over the remaining randomness) that $h_v(x_i)$ is contained in this fixed list is at most $\frac{\ell}{2\ell} \le 1/2$ (essentially, for any fixing of the remaining labels other than $x_i$, $h_v(x_i)$ could still be one of at least $2\ell$ equally likely labels; see Lemma A.1 in \cite{hanneke2024improved}).

    Averaging over the randomness in $S$ and $r$, we get that
    \begin{align*}
        \Pr \left[h_v(x_i) \notin \widehat{\mu}_S(x_i ) | E_i\right] \ge \frac{1}{2}.
    \end{align*}
    We thus obtain that
    \begin{align*}
        \Pr\left[h_v(s) \notin \widehat{\mu}_S(s ), s \neq x_1\right] 
        &= \sum_{i=2}^{d}\Pr[s = x_i] \Pr\left[h_v(x_i) \notin \widehat{\mu}_S(x_i )\right] \tag{test point independent of rest of the randomness} \\
        &\ge \sum_{i=2}^{d}\Pr[s = x_i] \Pr[E_i] \Pr\left[h_v(x_i) \notin \widehat{\mu}_S(x_i )|E_i\right] \\
        &\ge \sum_{i=2}^{d}\frac{4\eps}{d-1}\left( 1-\frac{8\eps}{d-1}\right)^m = 4\eps\left( 1-\frac{8\eps}{d-1}\right)^m \\
        &\ge 4\eps \left(1-\frac{8m\eps}{d-1}\right) \tag{using $d \ge 2, \eps \le 1/16$, and Bernoulli's inequality} \\
        &\ge 3\eps,
    \end{align*}
    whenever $m \le \frac{d-1}{32\eps}$. 
    
    This implies the existence of a fixed $v^\star \in V$ (which depends only on $\mcA$) for which
    \begin{align*}
        \Pr\left[h_{v^\star}(s) \notin \widehat{\mu}_S(s ), s \neq x_1\right]\ge 3\eps,
    \end{align*}
    whenever $m \le \frac{d-1}{32\eps}$.

    Now consider the modified learner $\mcA'$, which on any input sample $S$, outputs $\mcA'(S, r)=\widehat{\nu}_S$ satisfying
    \begin{align*}
        \widehat{\nu}_S(x_1) = \{h_{v^\star}(x_1)\}, \qquad \widehat{\nu}_S(x_j) = \widehat{\mu}_{S}(x_j) \,\, \text{for $j \in \{2,\dots,d\}$}.
    \end{align*}
    where $\widehat{\mu}_{S}=\mcA(S, r)$. Observe that
    \begin{align}
        \E \left[\err_{\mcD_{v^\star}}(\widehat{\nu}_S)\right] = \Pr\left[h_{v^\star}(s) \notin \widehat{\nu}_S(s )\right] &\ge \Pr\left[h_{v^\star}(s) \notin \widehat{\nu}_S(s ), s \neq x_1\right] \nonumber \\
        &= \Pr\left[h_{v^\star}(s) \notin \widehat{\mu}_S(s ), s \neq x_1\right] \nonumber \\
        &\ge 3\eps, \label{eqn:realizable-list-lb-to-contradict}
    \end{align}
    by the definition of $\mcA'$ and the calculations above.

    Now suppose it were true that 
    \begin{align*}
        \Pr\left[\err_{\mcD_{v^\star}}(\widehat{\mu}_S) > \eps\right] \le \frac{1}{8}.
    \end{align*}
    By definition of $\mcA'$, the error of $\widehat{\nu}_S$ can only be smaller than the error of $\widehat{\mu}_S$. Therefore, we would get that
    \begin{align*}
        \Pr\left[\err_{\mcD_{v^\star}}(\widehat{\nu}_S) > \eps\right] \le \frac{1}{8}.
    \end{align*}
    But note also that $\widehat{\nu}_S$ never errs on $x_1$, and hence, its error can at most be the total mass on $x_2,\dots,x_d$, which is $8\eps$. Putting all this together, we get
    \begin{align*}
        \E\left[\err_{\mcD_{v^\star}}(\widehat{\nu}_S)\right] &\le \eps + \frac{1}{8} \cdot 8\eps = 2\eps,
    \end{align*}
    which contradicts \eqref{eqn:realizable-list-lb-to-contradict}. Thus, it must instead be true that
    \begin{align*}
        \Pr\left[\err_{\mcD_{v^\star}}(\widehat{\mu}_S) > \eps\right] > \frac{1}{8},
    \end{align*}
    whenever $m \le \frac{d-1}{32\eps}$. This in turn implies that whenever $\delta \le 1/8$, the sample complexity must be at least $\Omega(d/\eps)$.

    \paragraph{Necessity of $\Omega(\log(1/\delta)/\eps)$.} Note that since we assume $d^{2\ell-1}_{\mathrm{DS}} \ge \dds \ge 2$, there exist $h_1,\dots,h_{2\ell} \in \mcH$ and $x_0,x_1$ such that $h_1(x_0)=\dots=h_{2\ell}(x_0)$, but $h_1(x_1),\dots,h_{2\ell}(x_1)$ are $2\ell$ distinct labels.

    Let $S$ consist of $m$ copies of the point $(x_0, h_1(x_0))$. 
    For each $j \in [2\ell]$, define
    \begin{align*}
        p_j := \Pr_{r}\left[h_j(x_1) \notin \widehat{\mu}_S(x_1)\right] = \E_r\left[\Ind[h_j(x_1) \notin \widehat{\mu}_S(x_1)]\right]
    \end{align*}
    where $\widehat{\mu}_S = \mcA(S, r)$. Since $\widehat{\mu}_S$ predicts at most $\ell$ labels for $x_1$, we have that
    \begin{align*}
        \frac{1}{2\ell}\sum_{j=1}^{2\ell} p_j = \frac{1}{2\ell} \E_r \left[\sum_{j=1}^{2\ell}\Ind[h_j(x_1) \notin \widehat{\mu}_S(x_1)]\right] \ge \frac{1}{2}.
    \end{align*}
    This means that there exists $j^\star$ such that $p_{j^\star} \ge 1/2$. For such a $j^\star$, define the realizable distribution $\mcD$ as
    \begin{align*}
        \mcD((x_0, h_{j^\star}(x_0))) = 1-4\eps, \qquad \mcD((x_1, h_{j^\star}(x_1))) = 4\eps.
    \end{align*}
    Now, let $E$ be the event that the realized sample $S$ consists of $m$ copies of $(x_0, h_{j^\star}(x_0))$. We have that $\Pr[E] = (1-4\eps)^m$. Let $B$ be the event that the realized randomness $r$ of the learner satisfies $h_{j^\star}(x_1) \notin \widehat{\mu}_S(x_1)$. We have that $\Pr[B]=p_{j^\star}\ge 1/2$. On the event $E \cap B$, the test error of the learner is at least the chance of sampling $x_1$, which is $4\eps > \eps$. Thus,
    \begin{align*}
        \Pr_{S, r}\left[\err_{\mcD}(\mcA(S,r)) > \eps\right] &\ge \Pr[E \cap B] \ge \frac{1}{2}(1-4\eps)^m.
    \end{align*}
    So, it must be the case that
    \begin{align*}
        \frac{1}{2}(1-4\eps)^m \le \delta \implies m = \Omega(\log(1/\delta)/\eps).
    \end{align*}
\end{proof}

\subsection{Agnostic Case}
\label{sec:agnostic-list-learning}

\subsubsection{Agnostic List Learning Upper Bound}
\label{sec:agnostic-list-learning-upper-bound}

In this section, we will derive a novel agnostic list learner that has optimal sample complexity (upto polylogarithmic factors). Our learner will be constructed in three steps, where the first two steps mirror those in the three-step pipeline of \cite{cohen2025sample}, but the third step is altogether different. We restate the sample complexity guarantee for the learner that we will construct for convenience.

\AgnosticList*
\begin{proof}%
    Fix $\mcD$ to be any arbitrary distribution over $\mcX \times \mcY$. We follow the 3-Step proof structure of \cite{cohen2025sample}, but instantiate a different Step 3.

    \paragraph{Step 1: Approximation by Finite List Cover.}
    Consider any sample $S=\{(x_1,y_1),\dots,(x_m,y_m)\}$ that is realizable by $\mcH$. By a standard minimax argument similar to that used in the proof of \Cref{thm:sample-compression}, there exists a compression function $\kappa$ which maps $S$ to a subsample $S' \in S^{k_1(m)}$, where $k_1(m) = O(\dds \log m)$, and a list reconstruction function $\rho$ with list size $|\rho(S')(x)| \le O(\ell \log m)$ for all $x \in \mcX$, such that $y_i \in \rho(S')(x_i)$ for every $i \in [m]$.

    To see this, fix any distribution $\mcP$ over $S$. Since $S$ is realizable by $\mcH$, the distribution $\mcP$ is also realizable by $\mcH$. Now consider the deterministic one-inclusion graph predictor $\mcA$ given in Algorithm 1 in \cite{charikar2023characterization}, which was also considered in the proof of \Cref{corollary:list-near-optimal-sample-complexity}. Similar to that proof (see also, e.g., the proof of Proposition 15 in \cite{brukhim2022characterization}), by a leave-one-out + exchangeability argument, together with the $\ell$-outdegree bounds given in \eqref{eqn:outdegree-bound}, \eqref{eqn:ell-outdegree-2}, it holds that
    \begin{align*}
        \Pr_{T \sim \mcP^{3\dds}, (x,y) \sim \mcP}[y \notin \mcA(T)(x)] \le \frac{\dds}{3\dds+1} \le 1/3.
    \end{align*}

    By the minimax theorem \cite{v1928theorie}, there exists a distribution $\mcQ$ over sequences $T \in S^{3\dds}$, such that for every fixed $(x_i,y_i) \in S$,
    \begin{align*}
        \Pr_{T \sim \mcQ}[y_i \notin \mcA(T)(x_i)] \le 1/3.
    \end{align*}
    Consider sampling $j$ sequences $T_1,\dots,T_j$ i.i.d.\ from $\mcQ$, for $j=O(\log m)$. We have that
    \begin{align*}
        \Pr_{T_1,\dots,T_j \sim \mcQ^j}\left[y_i \notin \bigcup_{r=1}^j \mcA(T_r)(x_i)\right] \le 3^{-j} \le 1/2m.
    \end{align*}
    By a union bound over $i \in [m]$ and an application of the probabilistic method, we get that there exist $T_1,\dots,T_j$ such that for every $i \in [m]$ simultaneously, it holds that $y_i \in \bigcup_{r=1}^j \mcA(T_r)(x_i)$. We thus set $\kappa(S)=S'=(T_1,\dots,T_j)$, and $\rho(S')(x)=\bigcup_{r=1}^j \mcA(T_r)(x)$. Note that $|\rho(S')(x)| \le O(\ell \log m)$ as required.

    Given this, consider obtaining a sample $S_1 \sim \mcD^{n_1}$ from the data distribution $\mcD$, for $n_1$ to be specified later. %
    Define the finite list cover family:
    \begin{align*}
        \mcF(S_1) := \{\rho(S'): S'=(T_1,\dots,T_{O(\log n_1)}), T_r \in S_1^{O(\dds)} \; \forall r\}.
    \end{align*}
    We have that $|\mcF(S_1)| %
    = n_1^{O(\dds \log(n_1))}$, and every member of $\mcF(S_1)$ is a list function mapping to a list of size $O(\ell \log n_1)$. Next, fix any $h \in \mcH$, and let $S_1(h)=\{(x,y) \in S_1: y = h(x)\}$, so that $S_1(h)$ is realizable by $h$. By our compression argument above, we have that there exists $\mu_h \in \mcF(S_1)$ such that $y \in \mu_h(x)$ for every $(x,y) \in S_1(h)$. In other words, $\mu_h$ has zero empirical loss on $S_1$ under the loss function
    \begin{align*}
        \Ind[h(x)=y, y \notin \mu(x)].
    \end{align*}
    Then, by the standard ``compression implies generalization'' argument (e.g., Theorem 30.2 in \cite{shalev2014understanding}), we get that %
    with probability at least $1-\delta/3$ over the draw of $S_1$, there exists $\mu_h \in \mcF(S_1)$ such that
    \begin{align*}
        \Pr_{(x,y) \sim \mcD}[h(x)=y, y \notin \mu_h(x)] \le O \left(\frac{\dds \log^2(n_1) + \log(1/\delta)}{n_1} \right).
    \end{align*}
    The bound above holds for every $h \in \mcH$.

    \paragraph{Step 2: Multiplicative Weights.} Condition on the high-probability good event on the sample $S_1$ above, and let $\mcF(S_1)$ be the finite cover of list hypotheses constructed there. Now consider obtaining a fresh sample $S_2 = \{(x_1,y_1),\dots,(x_T,y_T)\} \sim \mcD^T$, for $T$ to be specified later. We run the multiplicative weights procedure given in Algorithm 3 in \cite{cohen2025sample}. Namely, we initialize $w_1(\mu)=1$ for every $\mu \in \mcF(S_1)$, and proceed in rounds $t=1,2,\dots,T$. At round $t$, we sample $\mu_t \sim p_t$, where $p_t$ is a distribution over $\mcF(S_1)$ satisfying $p_t(\mu) \propto w_t(\mu)$. Then, for every $\mu \in \mcF(S_1)$, the reward $r_t(\mu)$ is defined as
    \begin{align*}
        r_t(\mu) = \Ind\left[y_t \in \mu(x_t),\, y_t \notin \bigcup_{r < t} \mu_{r}(x_t)\right].
    \end{align*}
    We correspondingly update $w_{t+1}(\mu) = w_t(\mu) \exp(r_t(\mu)/2)$ for every $\mu \in \mcF(S_1)$. At the end of $T$ rounds, we obtain the list hypothesis $\nu$, defined as 
    \begin{align*}
        \nu(x) = \bigcup_{t=1}^{T-1}\mu_t(x).
    \end{align*}
    We have that $|\nu(x)| \le O(\ell T \log n_1)$. Furthermore, the analysis of the multiplicative weights procedure from \cite{cohen2025sample} extends verbatim to our setting. Namely, Theorem 3.4 in \cite{cohen2025sample} guarantees that %
    for every $\mu \in \mcF(S_1)$, with probability at least $1-\delta/3$ (over the draw of $S_2$ and the multiplicative weights procedure, which then results in $\nu$), the list hypothesis $\nu$ satisfies
    \begin{align*}
        \Pr_{(x,y) \sim \mcD}\left[y \in \mu(x),\, y \notin \nu(x)\right] \le O\left(\frac{\log(|\mcF(S_1)|)+\log(1/\delta)}{T}\right).
    \end{align*}

\paragraph{Step 3:} We now deviate significantly from the Step 3 used by \cite{cohen2025sample} which was based on constructing an agnostic sample compression scheme; instead, we use a different machinery which constructs a \textit{fractional} $\ell$-list learner having small agnostic transductive error, and then use the results of \cite{dughmi25transductive} to convert this to a PAC learner. In what follows, suppose for convenience that both the good events from Steps 1 and 2 have occurred, resulting in the menu $\nu$ constructed at the end of Step 2.

With respect to $\nu$, define the ``inside-menu'' loss $L^\nu_\mcD$ of a predictor as
\begin{align*}
    &L^\nu_\mcD(\mu) = \Pr_{(x,y) \sim \mcD}[y \in \nu(x),\, y \notin \mu(x)], \quad \text{for list predictors $\mu$}, \\
    &L^\nu_\mcD(h) = \Pr_{(x,y) \sim \mcD}[y \in \nu(x),\, y \neq h(x)], \quad \text{for single-label predictors $h$}.
\end{align*}
Thus, $L^\nu_\mcD$ only penalizes a predictor on points $(x,y)$ where $y \in \nu(x)$.
Our goal now in this step is to derive an agnostic learner for this loss.

\paragraph{Shattered Pairs.} Fix any unlabeled sequence $(x_1,\dots,x_n)$, and let $Y_i = \nu(x_i)$. Let $\star_1,\dots,\star_n$ be ``inert'' symbols which signify labels being outside the prescribed menu $\nu(x_i)$.

Let us now map every $h \in \mcH$ to the pattern $f_h:[n] \to \bigcup_{i \in [n]} \left(Y_i \cup \{\star_i\}\right)$ as
\begin{align}
    \label{eqn:mapped-partial-patter}
    f_h(i) := \begin{cases}
        h(x_i) & \text{if $h(x_i) \in Y_i$}, \\
        \star_i & \text{otherwise}.
    \end{cases}
\end{align}
Then, consider the set of all such (distinct) patterns realized by $\mcH$ on the unlabeled sequence:
\begin{align*}
    F := \{f_h:h \in \mcH\} \subseteq \prod_{i \in [n]} \left(Y_i \cup \{\star_i\}\right).
\end{align*}
We say that a pair $(I, (A_i)_{i \in I})$, where $I \subseteq [n]$ and $A_i \in \binom{Y_i}{\ell+1}$, is shattered by $G \subseteq F$ if 
\begin{align*}
    \prod_{i \in I} A_i \subseteq G|_I.
\end{align*}
Now, for any $G \subseteq F$, let $s(G)$ be the number of distinct pairs $(I, (A_i)_{i \in I})$ that are shattered by $G$. We define $s(\emptyset)=0$, and include the pair corresponding to $I=\emptyset$ in the count of $s(G)$ for every non-empty $G$; hence, $s(G) \ge 1$ for all $G \neq \emptyset$. 

Note that $s$ is monotonic, i.e., $s(G) \le s(G')$ for $G \subseteq G'$. Moreover, since we never include an inert symbol $\star_i$ in $A_i$, any pair $(I, (A_i)_{i \in I})$ that is shattered by $F$ must satisfy that $|I| \le \dnat$. Thus, if every $|Y_i| \le p$, we have that
\begin{align}
    \label{eqn:s(F)-upper-bound}
    s(F) \le S_n := \sum_{j=0}^{\min(n,\dnat)} \binom{n}{j}{\binom{p}{\ell+1}}^{j}.
\end{align}
(Here, we use the convention that $\binom{a}{b} = 0$ if $a < b$.)

Later, we will also be interested in restrictions of the realizable patterns on $x_1,\dots,x_n$, when the label on $x_i$ is revealed to be $y_i$. Towards this, for $G \subseteq F$, $i \in [n]$ and $y \in Y_i$, define
\begin{align*}
    G_{i,y} := \{g \in G: g(i)=y\}.
\end{align*}
These restrictions satisfy the following useful inequality:
\begin{claim}[Restriction Inequality]
    \label{claim:restriction-inequality}
    For any $G \subseteq F$ and $i \in [n]$,
    \begin{align*}
        \sum_{y \in Y_i} s(G_{i,y}) \le \ell s(G).
    \end{align*}
\end{claim}
\begin{proof}
    The claim is vacuous for $G=\emptyset$. Let $W=(I, (A_j)_{j \in I})$ range over all distinct pairs satisfying $I \subseteq [n], i \notin I$ and $A_j \in \binom{Y_j}{\ell+1}$, including the empty pair. For any $W$, define
    \begin{align*}
        T_W = \{y \in Y_i: W \text{ is shattered by $G_{i, y}$}\}.
    \end{align*}
    Note that any pair $(I, (A_j)_{j \in I})$ shattered by $G_{i,y}$ must necessarily satisfy that $i \notin I$. Thus, double counting gives us that
    \begin{align}
        \label{eqn:double-counting}
        \sum_{y \in Y_i} s(G_{i,y}) = \sum_{W} |T_W|.
    \end{align}
    Now, if $|T_W| > 0$, then $W=(I, (A_j)_{j \in I})$ is shattered by $G$. Moreover, for every $A \in \binom{T_W}{\ell+1}$, observe that $G$ shatters $(I \cup \{i\}, (A_j)_{j \in I} \cup (A))$. Furthermore, all these pairs are distinct. Thus, we get that
    \begin{align}
        \label{eqn:s(G)-lb}
        s(G) \ge \sum_{W} \left(\Ind[|T_W| > 0] + \binom{|T_W|}{\ell+1}\right).
    \end{align}
    Now, for any integer $r \ge 0$, it holds that
    \begin{align*}
        r \le \ell \left(\Ind[r > 0] + \binom{r}{\ell+1}\right).
    \end{align*}
    Indeed, this is immediate for $0 \le r \le \ell$; for $r \ge \ell+1$, we use that $\binom{r}{\ell+1} \ge r-\ell$. Summing the above inequality over all $r=|T_W|$, and combining \eqref{eqn:double-counting}, \eqref{eqn:s(G)-lb} proves the claim.
\end{proof}

\paragraph{Thinned-version Spaces:} We will now employ an online-to-batch style analysis, with elements inspired from \cite{moran2023list} and \cite{alon2022theory}.

Given an ordered, labeled sequence $\{(x_1,y_1),\dots,(x_n,y_n)\}$, fix $F$ corresponding to the unlabeled sequence $(x_1,\dots,x_n)$. We will process these points sequentially: for each $t \in [n+1]$, define
\begin{align*}
    I_{t-1} := \{r < t: y_{r} \in Y_{r}\}.
\end{align*}
to be the data points seen so far whose labels lie in the menu. 

Now, for every $r \in I_{t-1}$, consider independently drawing $\xi_r \sim \mathrm{Ber}(\gamma)$, for a parameter $\gamma$ to be chosen suitably later, and let 
\begin{align*}
    A_t := \{r \in I_{t-1}: \xi_r = 1\}.
\end{align*}
For any possible retained set $A \subseteq I_{t-1}$, we have that $\Pr[A_t = A] = \gamma^{|A|}(1-\gamma)^{|I_{t-1}|-|A|}$. Define the version space induced by $A$ to be
\begin{align*}
    V_t(A) := \{f \in F: f(r) = y_{r} \text{ for every $r \in A$}\}.
\end{align*}
We can now define a deterministic fractional $\ell$-list predictor $p_t$  supported on $\nu(x_{t})$ as follows:
\begin{align*}
    C_t := \E[s(V_t(A_t))], \qquad B_{t,y} := \E[s(V_t(A_t)_{t, y})], \qquad p_{t,y} := \frac{B_{t,y}}{C_t} \text{ for $y \in Y_{t}$}.
\end{align*}
We first observe that every $C_t > 0$; to see this, note that there is a positive probability that $A_t=\emptyset$, in which case $V_t(A_t)=F$, and $s(F) \ge 1$. Furthermore, $B_{t,y} \le C_t$ by monotonicity of $s$, and hence $0 \le p_{t,y} \le 1$. Lastly, 
\begin{align*}
    \sum_{y \in Y_t} B_{t,y} = \E \left[ \sum_{y \in Y_t} s(V_t(A_t)_{t,y})\right] \le \E \left[ \ell s(V_t(A_t))\right] = \ell C_t \tag{using \Cref{claim:restriction-inequality}}.
\end{align*}
Thus, every $p_t$ is a well-defined fractional $\ell$-list. Furthermore, note that the predictor $p_t$ is obtained without looking at the label $y_t$.

Now, for a fractional $\ell$-list predictor $q$ supported on $\nu$, define
\begin{align*}
    \lambda^\nu(q, (x,y)) = \Ind[y \in \nu(x)](1-q_y(x)).
\end{align*}
In words, $q_y(x)$ is the probability that $q$ assigns the label $y$, provided that $y \in \nu(x)$, and $\lambda^\nu(q, (x,y))$ measures the probability of the complement.

For a (normal) function $f$ and (normal) list function $\mu$, we continue to have the natural definitions: $\lambda^\nu(f, (x,y)) =  \Ind[y \in \nu(x), y \neq f(x)]$ and $\lambda^\nu(\mu, (x,y)) =  \Ind[y \in \nu(x), y \notin \mu(x)]$. Thus, for a predictor $q$ --- fractional or otherwise --- we have that
\begin{align*}
    L^\nu_\mcD(q)=
\E_{(x,y)\sim\mcD}[\lambda^\nu(q,(x,y))].
\end{align*}

Now, we set the parameter $\gamma$ for the Bernoullis $\xi_r$: let $\gamma = 1-e^{-\eta}$, for $\eta = \min\{1,\sqrt{\log(S_n)/n}\}$. For this setting of $\gamma$, the following claim holds:

\begin{claim}
    \label{claim:thinned-version-space}
    For every $f \in F$, it holds that
    \begin{align*}
        \sum_{t=1}^n\lambda^\nu(p_t, (x_t, y_t)) - \sum_{t=1}^n\lambda^\nu(f, (x_t, y_t)) \le 3\left(\sqrt{n\log(S_n)}+\log(S_n)\right)
    \end{align*}
\end{claim}
\begin{proof}
    Suppose that $y_t \in Y_t$; in this case, at time $t+1$, $t$ would have been included in $A_{t+1}$ with probability $\gamma$; this gives us that
\begin{align*}
    &\E[s(V_{t+1}(A_{t+1}))] = (1-\gamma)\E[s(V_t(A_t))]+ \gamma\,  \E[s(V_{t}(A_{t})_{t, y_t})] \implies C_{t+1} = (1-\gamma)C_t + \gamma B_{t, y_t} \\
    \implies \quad & C_{t+1} = C_t(1-\gamma(1-p_{t, y_t})) = C_t(1-\gamma \lambda^\nu(p_t, (x_t, y_t))).
\end{align*}
    On the other hand, if $y_t \notin Y_t$, we have that $C_{t+1}=C_t$; since $\lambda^\nu(p_t, (x_t,y_t))=0$ in this case, the equality above still holds. Then, using that $C_1=s(F)$ and unrolling the recurrence, we get
    \begin{align*}
        &C_{n+1} = s(F)\prod_{t=1}^n(1-\gamma \lambda^\nu(p_t, (x_t, y_t))) \implies \log(C_{n+1}) = \log(s(F))+\sum_{t=1}^n \log (1-\gamma \lambda^\nu(p_t, (x_t, y_t))) \\
        \implies \quad & \log(C_{n+1}) \le \log(s(F))-\gamma \sum_{t=1}^n\lambda^\nu(p_t, (x_t, y_t)). \tag{using $\log(1-x) \le -x$}
    \end{align*}
    Now fix any $f \in F$. We have that $f \in V_{n+1}(A_{n+1})$ if and only if at all $t \in [n]$ satisfying $y_t \in Y_t$, $y_t \neq f(x_t)$, we chose to not include $t$ in $A_{n+1}$, meaning that $\xi_t=0$. Thus, we get that
    \begin{align*}
        \Pr[f \in V_{n+1}(A_{n+1})] = (1-\gamma)^{\sum_{t=1}^n \lambda^\nu(f, (x_t, y_t))}.
    \end{align*}
    On this event, we have that $V_{n+1}(A_{n+1})$ is non-empty, and hence $s(V_{n+1}(A_{n+1})) \ge 1$. Thus,
    \begin{align*}
        C_{n+1} = \E[s(V_{n+1}(A_{n+1}))] \ge (1-\gamma)^{\sum_{t=1}^n \lambda^\nu(f, (x_t, y_t))}.
    \end{align*}
    Combining with the previous reasoning, we thus get
    \begin{align*}
        &\gamma \sum_{t=1}^n\lambda^\nu(p_t, (x_t, y_t))  \le \log(s(F)) +\log\left(\frac{1}{1-\gamma}\right)\sum_{t=1}^n \lambda^\nu(f, (x_t, y_t)) \\
        \implies \quad & \sum_{t=1}^n\lambda^\nu(p_t, (x_t, y_t)) - \sum_{t=1}^n\lambda^\nu(f, (x_t, y_t)) \le \frac{\log(s(F))}{\gamma} + \left(\frac{1}{\gamma}\log\left(\frac{1}{1-\gamma}\right)-1\right)\sum_{t=1}^n \lambda^\nu(f, (x_t, y_t)) \\
        \implies \quad & \sum_{t=1}^n\lambda^\nu(p_t, (x_t, y_t)) - \sum_{t=1}^n\lambda^\nu(f, (x_t, y_t)) \le \frac{\log(s(F))}{1-e^{-\eta}} + \left(\frac{\eta}{1-e^{-\eta}}-1\right)\sum_{t=1}^n \lambda^\nu(f, (x_t, y_t)).
    \end{align*}
    Now, for $\eta \in (0,1]$, $\frac{1}{1-e^{-\eta}} \le \frac{2}{\eta}$, and $\frac{\eta}{1-e^{-\eta}} -1 \le \eta$; substituting above, and noting that $\sum_{t=1}^n \lambda^\nu(f, (x_t, y_t)) \le n$, we get
    \begin{align*}
        \sum_{t=1}^n\lambda^\nu(p_t, (x_t, y_t)) - \sum_{t=1}^n\lambda^\nu(f, (x_t, y_t)) &\le \frac{2\log(s(F))}{\eta} + \eta n \\
        &\le 3\left(\sqrt{n \log(S_n)}+\log(S_n)\right). \tag{substituting $\eta$ and using $s(F) \le S_n$}
    \end{align*}
    \end{proof}
    
    \paragraph{Constructing a fractional list learner.} %

    Suppose we are given a labeled sample $\{(x_1,y_1),\dots,(x_n, y_n)\}$, and we wish to make a prediction on an unlabeled query point $x$. We append the query point to the unlabeled sequence $(x_1,\dots,x_n,x)$. Now, we consider all $(n+1)!$ permutations of this sequence. Suppose that the entire sequence in a given permutation is of the form
    \begin{align*}
        x_{i_1}, x_{i_2},\dots,x_{i_{t}}, x, x_{i_{t+1}},\dots,x_{i_n}.
    \end{align*}
    Then, we obtain the fractional list predictor based on thinning the version space by processing the labeled sequence $(x_{i_1}, y_{i_1}),\dots,(x_{i_{t}}, y_{i_t})$ that precedes the query point $x$ in this order. We record the fractional list prediction made for $x$ in every such permutation, and return the average fractional prediction over all permutations to be the final fractional list prediction made for $x$. We will denote this average fractional list prediction derived from the labeled points $S$ for a query point $x$ as $Q(x; S)$. Note that $Q(x;S)$ is deterministic, symmetric in the indexed training sample $S$, and never uses any label for the query point.

    We now show that this fractional list predictor has low transductive error. Namely, for $S=\{(x_1,y_1),\dots,(x_n,y_n)\}$, let $S_{-i}$ include all the points in $S$ except $(x_i, y_i)$. Then, using the fact that the loss $\lambda^\nu$ is linear in its first argument, we obtain that
    \begin{align}
        \label{eqn:symmetrizing-over-permutations}
        \sum_{i=1}^n \lambda^\nu(Q(x_i; S_{-i}), (x_i, y_i)) = \frac{1}{n!}\sum_{\pi} \sum_{t=1}^n \lambda^\nu(p^{\pi}_t, (x_{\pi(t)}, y_{\pi(t)})),
    \end{align}
    where $(p^\pi_1,\dots,p^\pi_n)$ are the fractional lists generated while processing the labeled sequence $$(x_{\pi(1)}, y_{\pi(1)}), \dots, (x_{\pi(n)}, y_{\pi(n)})$$ according to the description of the thinned-version-space process above.

    Now for any $h \in \mcH$, we have from \Cref{claim:thinned-version-space} that
    \begin{align*}
        \sum_{t=1}^n \lambda^\nu(p^\pi_t, (x_{\pi(t)}, y_{\pi(t)})) &\le \sum_{t=1}^n \lambda^\nu(f_h, (x_{\pi(t)}, y_{\pi(t)})) + 3\left(\sqrt{n \log(S_n)}+\log(S_n)\right) \\
        &= \sum_{i=1}^n \lambda^\nu(h, (x_{i}, y_{i})) + 3\left(\sqrt{n \log(S_n)}+\log(S_n)\right).
    \end{align*}
    Averaging over all permutations, and using \eqref{eqn:symmetrizing-over-permutations}, we get that
    \begin{align}
        \label{eqn:transductive-error-bound}
        \frac{1}{n}\sum_{i=1}^n \lambda^\nu(Q(x_i; S_{-i}), (x_i, y_i)) &\le \frac{1}{n}{\sum_{i=1}^n \lambda^\nu(h; (x_i, y_i))} + 3\left(\sqrt{\frac{ \log(S_n)}{n}}+\frac{\log(S_n)}{n}\right),
    \end{align}
    Thus, the transductive error of the fractional list predictor $Q$ is competitive with the transductive error of any fixed $h \in \mcH$; namely, this is an agnostic transductive error guarantee for $Q$.

    We are then in a position to directly invoke a result of \cite{dughmi25transductive}, which relates agnostic and transductive error rates. We state their result in a form that is useful for us:
    \begin{theorem}[Theorem 2 in \cite{dughmi25transductive}]
        \label{thm:dughmi-agnostic-transductive}
        Suppose a symmetric transductive learning rule in the agnostic setting has excess transductive error at most $\rho(n)$ over the best hypothesis from a class $\mcH$, for a loss bounded in $[0,1]$, where $\rho$ is non-increasing. If $\rho(n) \le \eps/4$, then with
        \begin{align*}
            n + O\left(\frac{1}{\eps^2}\log\left( \frac{1}{\eps \delta}\right)\right)
        \end{align*}
        i.i.d.\ samples from any given distribution, with probability $1-\delta$, one can obtain a predictor whose population loss is within $\eps$ of the best hypothesis from the class $\mcH$.
    \end{theorem}
    \cite{dughmi25transductive} state their original theorem for the standard multiclass loss; however, their proof for the agnostic result above only uses boundedness of the loss, symmetry of the transductive learning rule and non-increasingness of the excess transductive error $\rho(n)$. Therefore, their proof applies verbatim to the loss $\lambda^\nu \in [0,1]$, and \eqref{eqn:transductive-error-bound} is precisely its transductive premise.

    In order to precisely pin down $\rho(n)$, recall that
    \begin{align*}
        S_n = \sum_{j=0}^{\min(n,\dnat)} \binom{n}{j}{\binom{p}{\ell+1}}^{j} \implies \frac{\log S_n}{n} \le \frac{\dnat\log(en\binom{p}{\ell+1}/\dnat)}{n} := g(n).
    \end{align*}
    Thus, from the transductive excess error bound \eqref{eqn:transductive-error-bound}, we can choose $\rho(n)$ to be the non-increasing function
    \begin{align*}
        \rho(n) = \begin{cases}
            1 & \text{if $n < \dnat$}, \\
            \min(1, 3(\sqrt{g(n)}+g(n))) & \text{otherwise}.
        \end{cases}
    \end{align*}
    Solving for $\rho(n)\le \eps/4$ and instantiating \Cref{thm:dughmi-agnostic-transductive} immediately gives us the following:
    \begin{theorem}[Fractional List Learner]
        \label{thm:fractional-list-learner}
        Let $\nu$ be a menu of size at most $p$ fixed independently of the training data. There is a learner that uses
        \begin{align*}
            n = \widetilde{O} \left(\frac{\ell\dnat + \log(1/\delta)}{\eps^2}\right)
        \end{align*}
        i.i.d.\ samples from any distribution $\mcD$, and outputs a fractional $\ell$-list predictor $q$ supported on $\nu$, such that with probability at least $1-\delta$, it holds that $L^\nu_\mcD(q) \le \inf_{h \in \mcH}L^\nu_\mcD(h)+\eps$.
    \end{theorem}
    \paragraph{Rounding the fractional list learner.} We now wish to convert any fractional $\ell$-list $q(x)$ output by the learner to an actual list of at most $\ell$ labels, while not degrading error. We could sample every label $y$ with probability $q(x)_y$; however, this might yield more than $\ell$ labels.

    To overcome this, we implement the following rounding: for any $x$, consider $\nu(x)$ written down in some fixed order $\nu(x)=(y_1,\dots,y_r)$. Suppose that the fractional predictor assigns masses
    \begin{align*}
        0 \le q_{y_j}(x) \le 1, \qquad \sum_{j=1}^r q_{y_j}(x) \le \ell.
    \end{align*}
    Place the following intervals on the line:
    \begin{align*}
        J_{y_j}(x) := \left[\sum_{s < j}q_{y_s}(x), \sum_{s \le j}q_{y_s}(x)\right).
    \end{align*}
    That is, the intervals are disjoint, $J_{y_j}(x)$ has length exactly $q_{y_j}(x)$, and all intervals lie in 
    \begin{align*}
        \left[0, \sum_{y \in \nu(x)}q_y(x)\right).
    \end{align*}
    We now draw $u \in (0,1)$ uniformly at random, and consider the $\ell$ points
    \begin{align*}
        u, u+1, \dots, u+\ell-1.
    \end{align*}
    The rounded list $\mu_u(x)$ contains the label $y$ precisely when one of these points falls in $J_y(x)$, i.e.,
    \begin{align}
        \label{eqn:rounding}
        \mu_u(x) := \left\{y \in \nu(x): u+j \in J_y(x) \text{ for some $j \in \{0,1,\dots,\ell-1\}$}\right\}.
    \end{align}
    Since the intervals are pairwise disjoint, the rounded list $\mu_u(x)$ has at most $\ell$ labels always.

    We now argue that each label $y \in \nu(x)$ has the desired marginal. Suppose $J_y(x)=[b, b+q_{y}(x))$. Observe that the points $u,u+1,\dots,u+\ell-1$ hit this interval only when $u$ belongs to the interval $J_y(x)$ modulo 1. Further, since $|J_y(x)| = q_y(x) \le 1$, reduction modulo 1 preserves its length; even if $J_y(x)$ crosses an integer, it gets split across two intervals (e.g., $[1.7, 2.3)$ gets split into $[0.7, 1) \cup [0, 0.3)$). Thus, we have that
    \begin{align*}
        \Pr_{u}[y \in \mu_u(x)] = q_y(x).
    \end{align*}
    Thus, when $y \in \nu(x)$, we have that
    \begin{align*}
        \E_u[\Ind[y \notin \mu_u(x)]] = 1-q_y(x) = \lambda^\nu(q, (x,y)).
    \end{align*}
    Averaging over $(x,y) \sim \mcD$, we get
    \begin{align*}
        \E_u[L^\nu_\mcD(\mu_u)] &= \E_{u}\E_{(x,y) \sim \mcD}[\Ind[y \in \nu(x), y \notin \mu_u(x)]] \\
        &=\E_{(x,y) \sim \mcD}\E_u[\Ind[y \in \nu(x), y \notin \mu_u(x)]] \\
        &= \E_{(x,y) \sim \mcD} [\lambda^\nu(q,(x,y))] = L^\nu_{\mcD}(q).
    \end{align*}
    In particular, this means that for some $u \in (0,1)$, it holds that
    \begin{align*}
        L^\nu_\mcD(\mu_u) \le L^\nu_{\mcD}(q).
    \end{align*}
    In order to determine the value of $u$ to use, we draw a validation sample $V \sim \mcD^v$, and apply a uniform convergence bound.

    Concretely, for any $u \in (0,1)$, define the loss function
    \begin{align*}
        g_u((x,y)) = \lambda^\nu(\mu_u, (x,y)).
    \end{align*}
    Let
    \begin{align}
        \widehat{L}^\nu_{V}(g_u) = \frac{1}{v}\sum_{(x,y) \in V} g_u((x,y)) .
    \end{align}
    We will choose $$\widehat{u}=\argmin_{u \in (0,1)} \widehat{L}^\nu_{V}(g_u),$$
    and set $\widehat{\mu}=\mu_{\widehat{u}}$.
    
    Note that for any fixed $(x,y)$, if $y\notin \nu(x)$, then $g_u((x,y))=0$ for all $u$. Otherwise, $g_u((x,y))=1 \iff u \notin J_y(x) \text{ (modulo 1)}$. As reasoned above, the set $J_y(x) \text{ (modulo 1)}$ is one (circular) interval within $(0,1)$; in particular, $g_u((x,y))$ changes its value only at its (at most two) endpoints.

    The $v$ validation points correspondingly induce $O(v)$ such intervals (with $O(v)$ many endpoints). Thus, we have a bound on the growth function: as $u$ ranges over $(0,1)$, we have that the number of different binary patterns that $g_u$ can realize on any fixed $v$ points is at most $O(v)$. By a standard uniform convergence bound, choosing $v = O\left(\frac{\log(1/\eps)+\log(1/\delta)}{\eps^2}\right)$, and taking a union bound over the failure of the fractional list predictor, we obtain that
    \begin{align*}
        L^\nu_\mcD(\widehat{\mu}) \le \inf_{h \in \mcH}L^\nu_\mcD(h) + \eps.
    \end{align*}

    To summarize, Step 3 shows how we can draw a sample of total size $\widetilde{O} \left(\frac{\ell\dnat + \log(1/\delta)}{\eps^2}\right)$ (corresponding to the samples used by \Cref{thm:fractional-list-learner} to construct a fractional list learner, and the validation samples used above to round it) and obtain an $\ell$-list predictor that has small agnostic error with respect to the $L^\nu_\mcD$ loss.
    
    \paragraph{Putting it all together.} 
    We will draw independent batches of $n_1+T+n_3$ samples for the purposes of Steps 1, 2 and 3.
    
    Let $h^\star \in \mcH$ satisfy that $\err_{\mcD}(h^\star) \le \inf_{h \in \mcH} \err_\mcD(h) + \eps/4$. Condition on the good event in Step 1 over the sample $S_1 \sim \mcD^{n_1}$, which happens with probability at least $1-\delta/3$, to obtain a list cover $\mcF(S_1)$ with the guarantee that there exists $\mu^\star \in \mcF(S_1)$ for which
    \begin{align}
        \Pr[h^\star(x)=y, \, y \notin \mu^\star(x)] \le O \left(\frac{\dds \log^2(n_1) + \log(1/\delta)}{n_1} \right). \label{eqn:eqn1}
    \end{align}
    Next, condition further on the good event in Step 2 over the sample $S_2 \sim \mcD^{T}$ and the multiplicative weights procedure, which happens with probability at least $1-\delta/3$, and results in a list hypothesis $\nu$ with the guarantee
    \begin{align}
        \Pr[y \in \mu^\star(x), \, y \notin \nu(x)] \le O\left(\frac{\log(|\mcF(S_1)|)+\log(1/\delta)}{T}\right). \label{eqn:eqn2}
    \end{align}
    A union bound gives us that with probability at least $1-2\delta/3$ over Steps 1 and 2,
    \begin{align}
        \Pr[h^\star(x)=y, \, y \notin \nu(x)] &= \Pr[h^\star(x)=y, \, y \notin \nu(x), \, y \notin \mu^\star(x)] + \Pr[h^\star(x)=y, \, y \notin \nu(x), \, y \in \mu^\star(x)] \nonumber \\
        &\le \Pr[h^\star(x)=y, \, y \notin \mu^\star(x)]  + \Pr[y \notin \nu(x), \, y \in \mu^\star(x)] \nonumber \\
        &\le O \left(\frac{\dds \log^2(n_1) + \log(1/\delta)}{n_1} + \frac{\log(|\mcF(S_1)|)+\log(1/\delta)}{T} \right),\label{eqn:penultimate-step}
    \end{align}
    where in the last inequality, we combined \eqref{eqn:eqn1} and \eqref{eqn:eqn2}.
    
    Finally, condition on the good event in Step 3, which happens with probability at least $1-\delta/3$ over the draw of $S_3 \sim \mcD^{n_3}$, where $n_3 = \widetilde{O} \left(\frac{\ell\dnat + \log(1/\delta)}{\eps^2}\right)$, and results in an $\ell$-list predictor $\widehat{\mu}$ that satisfies
    \begin{align*}
        L^\nu_\mcD(\widehat{\mu}) \le \inf_{h \in \mcH}L^\nu_\mcD(h) + \eps/4.
    \end{align*}
    This further implies, by a similar chain of inequalities as in \cite[Theorem 3.5]{cohen2025sample}, that %
    
    \begin{align*}
        \err_{\mcD}(\widehat{\mu}) - \err_{\mcD}(h^\star) &\le \Pr[y \notin \nu(x), h^\star(x)=y] + \eps/4.
    \end{align*}
    Taking a union bound over the event in \eqref{eqn:penultimate-step}, we get that with probability at least $1-\delta$ over all the steps, it holds that
    \begin{align*}
        \err_{\mcD}(\widehat{\mu}) - \err_{\mcD}(h^\star) &\le O \left(\frac{\dds \log^2(n_1) + \log(1/\delta)}{n_1} + \frac{\log(|\mcF(S_1)|)+\log(1/\delta)}{T}\right) + \eps/4 \\
        &= O \left(\frac{\dds \log^2(n_1) + \log(1/\delta)}{n_1} + \frac{ \dds \log^2(n_1) +\log(1/\delta)}{T}\right) + \eps/4.
    \end{align*}
    We set $n_1$, $T$ such that each of the terms in the parentheses above is at most $\eps/4$. This requires setting
    \begin{align*}
        &n_1 = \widetilde{O}\left(\frac{\dds+\log(1/\delta)}{\eps}\right), \quad
        T = \widetilde{O}\left(\frac{\dds+\log(1/\delta)}{\eps}\right). %
    \end{align*}
    Recalling that $\err_{\mcD}(h^\star) \le \inf_{h \in \mcH}\err_{\mcD}(h) + \eps/4$, we then get
    \begin{align*}
        \err_{\mcD}(\widehat{\mu}) \le \inf_{h \in \mcH}\err_{\mcD}(h) + \eps,
    \end{align*}
    which is the required agnostic list learning guarantee. The total sample complexity is
    \begin{align*}
        n_1 + T + n_3 = \widetilde{O}\left(\frac{\dds}{\eps} + \frac{\ell d^\ell_{\mathrm{Nat}} + \log(1/\delta)}{\eps^2}\right).
    \end{align*}
\end{proof}

\subsubsection{Agnostic List Learning Lower Bound}
\label{sec:agnostic-list-learning-lower-bound}

In this section, we separately show lower bounds of $\Omega(\log(1/\delta)/\eps^2)$ and $\Omega(\ell \dnat/\eps^2)$ on the agnostic sample complexity of $\ell$-list learning. Together with our improved lower bound of $\Omega(\dds/\eps)$ from \Cref{thm:realizable-list-learning-lower-bound} for the realizable setting, this implies that the agnostic sample complexity of $\ell$-list learning is lower-bounded as
\begin{align}
    \label{eqn:final-agnostic-list-learning-lower-bound}
    \Omega\left(\frac{\dds}{\eps}+\frac{\ell \dnat + \log(1/\delta)}{\eps^2}\right).
\end{align}

\begin{theorem}[Agnostic List Learning Lower Bound]
    \label{thm:agnostic-list-learning-lower-bound}
    The following two statements hold:
    \begin{enumerate}
        \item[(1)] For every $\ell \ge 1$, there is a hypothesis class $\mcH$ having $\dnat(\mcH)=\dds(\mcH)=1$ such that for every $\ell$-list learner $\mcA$ (possibly using randomness $r$), $(\eps,\delta) \in \left(0,\frac{1}{8}\right)^2$ and $m \le \frac{\log(1/8\delta)}{64\eps^2}$, there is a distribution $\mcD$ for which
        \begin{align*}
            \Pr_{S \sim \mcD^m,r}\left[\err_{\mcD}(\mcA(S,r)) > \inf_{h \in \mcH} \err_{\mcD}(h)+\eps \right] > \delta.
        \end{align*}
        \item[(2)] For every $\ell,d \ge 1$, there is a hypothesis class $\mcH$ having $\dnat(\mcH)=d$ such that for every $\ell$-list learner $\mcA$ (possibly using randomness $r$), $\eps \in \left(0,\frac{1}{64}\right)$ and $m \le \frac{\ell d}{32768\eps^2}$, there is a distribution $\mcD$ for which
        \begin{align*}
            \Pr_{S \sim \mcD^m,r}\left[\err_{\mcD}(\mcA(S,r)) > \inf_{h \in \mcH} \err_{\mcD}(h)+\eps \right] > \frac{1}{10}.
        \end{align*}
    \end{enumerate}
\end{theorem}
\begin{proof}   
    We first prove part (1) of the theorem.
    
    \paragraph{Necessity of $\Omega(\log(1/\delta)/\eps^2)$.} Let $\mcX = \N$. Let $\mcY_0, \mcY_1$ be disjoint label sets with $|\mcY_0|=|\mcY_1|=\ell$, and set $\mcY = \mcY_0 \sqcup \mcY_1$. We set $\mcH$ to be the class that includes all functions from $\mcX \to \mcY_0$ as well as all functions from $\mcX \to \mcY_1$.
    We can readily see that $\dds(\mcH)=\dnat(\mcH)=1$.

    For a hidden bit $\theta \in \{0,1\}$, which we will draw uniformly at random, define a distribution over $\mcY$ as
    \begin{align*}
        p_\theta(y) = \begin{cases}
            \frac{1+4\eps}{2\ell} & \text{if $y \in \mcY_\theta$}, \\
            \frac{1-4\eps}{2\ell} & \text{otherwise}.
        \end{cases}
    \end{align*}
    Thus, the distribution $p_\theta$ favors labels in $\mcY_\theta$. Now, let $M$ be a suitably large integer to be chosen later. Conditioned on $\theta$, draw $F(1),\dots,F(M) \sim p_\theta$ i.i.d., and let $\mcD_{F}$ be the uniform distribution over
    \begin{align*}
        \{(x,F(x)): x \in [M]\}.
    \end{align*}
    For any given $(\theta, F)$, there is a hypothesis $h_{\theta, F} \in \mcY_{\theta}^\mcX$ that agrees with $F(x)$ whenever $F(x) \in \mcY_\theta$, and outputs any other fixed label from $\mcY_\theta$ otherwise. We have that
    \begin{align}
        \label{eqn:comparator-hypothesis-error}
        \inf_{h \in \mcH}\err_{\mcD_{F}}(h) \le \err_{\mcD_{F}}(h_{\theta, F}) = \E_{(x,F(x)) \sim \mcD_{F}}[\Ind[F(x) \notin \mcY_\theta]].
    \end{align}
    So, consider drawing $\theta$ uniformly at random from $\{0,1\}$, then drawing $F(1),\dots,F(M) \sim p_{\theta}$, and then drawing a training sample $S = \{(x_1,y_1),\dots,(x_m,y_m)\} \sim \mcD_{F}^m$. Fix $\mcA$ to be any arbitrary learner, possibly using randomness $r$. We will assume without loss of generality (e.g., by deduplication + padding), that the predictor $\mcA(S,r)$ output by the learner always predicts exactly $\ell$ distinct labels for any $x$; the error of any such possibly deduplicated + padded predictor can only be smaller than the original predictor.
    
    Let $T=(S, r)$ denote the full training ``transcript'' of the given learner $\mcA$, which includes the training sample $S$ and its internal randomness $r$. We will first show that the distributions of the transcript $T$ separately under $\theta=0$ and $\theta=1$ are very close. Namely, let $P^\theta(T)$ be the distribution of the transcript $T$ conditioned on the hidden bit being $\theta$. Since $r$ is independent of the training sample, we have that
    \begin{align*}
        \KL(P^0(T) || P^1(T)) = \KL(P^0(S) || P^1(S)).
    \end{align*}
    Now, condition on any fixed sequence $(x_1,\dots,x_m)$ of unlabeled samples, and consider the conditional distribution of the labels $y_1,\dots,y_m$. The label sequence is determined by the draws $F(x)$ for all the \textit{distinct} values of $x$ in the unlabeled sequence; furthermore, conditioned on the unlabeled sequence, these are independent draws from $p_\theta$. So, denoting by $m'$ the number of distinct values of $x$ in the unlabeled sequence, and observing that the unlabeled samples have the same distribution under both $\theta=0$ and $\theta=1$, we have that
    \begin{align*}
        \KL(P^0(S) || P^1(S)) &= \E_{x_1,\dots,x_m}\left[\KL(P^0(y_1,\dots,y_m|x_1,\dots,x_m) || P^1(y_1,\dots,y_m|x_1,\dots,x_m))\right] \\ 
        &= \E_{x_1,\dots,x_m}\left[ m' \cdot \KL(p_0 || p_1) \right] = \E[m'] \cdot \KL(p_0 || p_1) \le m \cdot \KL(p_0 || p_1) \\
        &= 4m\eps \log\left(\frac{1+4\eps}{1-4\eps}\right) \\
        &\le 64m\eps^2. \tag{since $\eps < 1/8,  \log\left(\frac{1+4\eps}{1-4\eps}\right) \le 16 \eps$}
    \end{align*}
    Thus, we have that
    \begin{align}
        \label{eqn:kl-upper-bound}
        \KL(P^0(T) || P^1(T)) \le 64m\eps^2.
    \end{align}
    Now let $\widehat{\theta}$ be any test based on $T$ for determining the hidden bit $\theta$. By the standard two-point testing lower bound, we have that
    \begin{align*}
        \frac{1}{2}\left(P^0(\widehat{\theta}=1)+P^1(\widehat{\theta}=0)\right) &\ge \frac{1-\mathrm{TV}(P^0(T), P^1(T))}{2} \\
        &\ge \frac{1}{4}\exp\left(-\KL(P^0(T) || P^1(T))\right). \tag{Bretagnolle–Huber inequality}
    \end{align*}
    Since $\theta$ is a uniformly random bit, the quantity on the left-hand side above is $\Pr[\widehat{\theta} \neq \theta]$, where the probability is over the entire random experiment. Hence, combining with \eqref{eqn:kl-upper-bound}, we get that
    \begin{align}
        \label{eqn:pr-test-wrong}
        \Pr[\widehat{\theta} \neq \theta] \ge \frac{1}{4}\exp(-64m\eps^2).
    \end{align}
    Thus, we have a lower bound on the error of any estimator $\widehat{\theta}$ for $\theta$. 

    We will now see how we can lower-bound the excess error of the predictor $\mcA(S,r)=\mcA(T)=\mu_T$ output by the learner over the error of the comparator hypothesis $h_{\theta, F}$ defined in \eqref{eqn:comparator-hypothesis-error}, by relating it to the error of a specific estimator $\widehat{\theta}$ that we will define ahead. 
    
    Namely, for any $x \in [M]$, consider the random variable
    \begin{align*}
        &Z_x := \Ind[F(x) \notin \mu_T(x)] - \Ind[F(x) \neq h_{\theta, F}(x)] = \Ind[F(x) = h_{\theta, F}(x)] - \Ind[F(x) \in \mu_T(x)], \\
        &Z := \frac{1}{M}\sum_{x=1}^M Z_x = \err_{\mcD_{F}}(\mu_T) - \err_{\mcD_{F}}(h_{\theta, F}).
    \end{align*}
    Now let $O \subseteq [M]$ be the observed samples in $S$ and let $U=[M] \setminus O$ be the unobserved samples. Define
    \begin{align*}
        C_0(T) := \sum_{x \in U} |\mu_T(x) \cap \mcY_0|, \qquad C_1(T) := \sum_{x \in U} |\mu_T(x) \cap \mcY_1|,
    \end{align*}
    Since we assumed that $\mu_T$ always predicts exactly $\ell$ distinct labels, we have that
    \begin{align}
        \label{eqn:C-0-C-1-sum}
        C_0(T) + C_1(T) = \ell |U|.
    \end{align}
    Then consider the test
    \begin{align*}
        \widehat{\theta} = \argmax_{s \in \{0,1\}} C_s(T).
    \end{align*}
    This test determines the hidden bit based on whichever of the label sets $\mcY_0$ or $\mcY_1$ is seen more in the predictions of $\mcA$ on the unseen samples. We expect this to be a decent estimator for $\theta$ if the learner $\mcA$ learns well, since labels from $\mcY_\theta$ are likelier to be seen under $\mcD_{F}$.

    Let us now condition on $\theta, T$, which completely determines the predictor $\mu_T$, the observed data $O$ and the unobserved data $U$. Note that while the labels on $O$ get determined, the labels on $U$ remain random; in fact, they are mutually independent draws from $p_\theta$.

    Let us also further condition on the event that $\widehat{\theta} \neq \theta$. Then, by definition of the estimator and \eqref{eqn:C-0-C-1-sum}, it must be the case that
    \begin{align}
        \label{eqn:unobserved-error-large}
        \sum_{x \in U} |\mu_T(x) \cap \mcY_{1-\theta}| \ge \frac{\ell |U|}{2}.
    \end{align}
    Now let us analyze the expected value of $Z_x$ for $x \in U$. We have that
    \begin{align*}
        \E[Z_x|\theta,T,\widehat{\theta} \neq \theta] &= \Pr[F(x) = h_{\theta, F}(x)] - \Pr[F(x) \in \mu_T(x)] \\
        &= \ell \cdot \frac{1+4\eps}{2\ell} - \left(\frac{1+4\eps}{2\ell}|\mu_T(x) \cap \mcY_{\theta}| + \frac{1-4\eps}{2\ell}|\mu_T(x) \cap \mcY_{1-\theta}|\right) \\
        &= \frac{1+4\eps}{2} - \left(\frac{1+4\eps}{2\ell}(\ell-|\mu_T(x) \cap \mcY_{1-\theta}|) + \frac{1-4\eps}{2\ell}|\mu_T(x) \cap \mcY_{1-\theta}|\right) \\
        &= \frac{4\eps}{\ell}|\mu_T(x) \cap \mcY_{1-\theta}|.
    \end{align*}
    We therefore obtain that
    \begin{align*}
        \E[Z|\theta,T,\widehat{\theta} \neq \theta] &= \frac{1}{M} \left(\sum_{x \in U}\E[Z_x|\theta,T,\widehat{\theta} \neq \theta] + \sum_{x \in O}\E[Z_x|\theta,T,\widehat{\theta} \neq \theta]\right) \\
        &=\frac{1}{M} \left(\frac{4\eps}{\ell}\sum_{x \in U}|\mu_T(x) \cap \mcY_{1-\theta}| + \sum_{x \in O}\E[Z_x|\theta,T,\widehat{\theta} \neq \theta]\right) \\
        &\ge \frac{1}{M}\left(2\eps|U|-|O|\right) \tag{using \eqref{eqn:unobserved-error-large} and $Z_x \ge -1$} \\
        &= \frac{1}{M}\left(|U|(2\eps+1)-M\right) \tag{since $|O|=M-|U|$} \\
        &\ge \frac{1}{M}\left((M-m)(2\eps+1)-M\right) = 2\eps\left(1-\frac{m}{M}\right)-\frac{m}{M} \tag{since $|U| \ge M-m$}.
    \end{align*}
    If $M \ge \frac{8m}{\eps}$, then the last quantity above is at least $3\eps/2$ in the range $\eps \in (0,1/8)$, giving that
    \begin{align}
        \label{eqn:conditional-exp-lb}
        \E[Z|\theta,T,\widehat{\theta} \neq \theta] \ge \frac{3\eps}{2}.
    \end{align}
    But conditioned on $\theta, T, \widehat{\theta} \neq \theta$, we can also apply Hoeffding's inequality to the mutually independent random variables $Z_x$ for $x \in U$, each of which is bounded in the range $[-1,1]$, to obtain
    \begin{align*}
        \Pr\left[\sum_{x \in U}\left(Z_x - \E[Z_x]\right) \le -M\eps/2 ~\Big|~ \theta, T,  \widehat{\theta} \neq \theta\right] \le \exp\left(-\frac{M^2\eps^2}{8|U|}\right) \le \exp\left(-\frac{M\eps^2}{8}\right).
    \end{align*}
    Next, observe that conditioned on $\theta, T, \widehat{\theta} \neq \theta$, 
    \begin{align*}
        Z \le \eps \implies Z - \E[Z] \le -\eps/2 \implies \sum_{x \in [M]}\left(Z_x - \E[Z_x]\right) \le -M\eps/2 \implies \sum_{x \in U}\left(Z_x - \E[Z_x]\right)\le -M\eps/2,
    \end{align*}
    where in the first implication, we used the lower bound on the conditional expectation from \eqref{eqn:conditional-exp-lb}, and in the last implication, we used that conditioned on $\theta$ and $T$, the random variable $Z_x$ for $x \in O$ is determined.

    Therefore, we get that
    \begin{align*}
        \Pr[Z \le \eps | \theta, T, \widehat{\theta} \neq \theta] \le \Pr\left[\sum_{x \in U}\left(Z_x - \E[Z_x]\right) \le -M\eps/2 ~\Big|~ \theta, T,  \widehat{\theta} \neq \theta\right] \le \exp\left(-\frac{M\eps^2}{8}\right).
    \end{align*}
    If $M \ge \frac{8\log(4)}{\eps^2}$, then we get that
    \begin{align*}
        \Pr[Z \le \eps | \theta, T, \widehat{\theta} \neq \theta] \le \frac{1}{4}.
    \end{align*}
    Averaging over $\theta, T$, we obtain
    \begin{align*}
        \Pr[Z \le \eps | \widehat{\theta} \neq \theta] \le \frac{1}{4} \implies \Pr[Z > \eps | \widehat{\theta} \neq \theta] \ge \frac{3}{4}.
    \end{align*}
    Finally, this gives
    \begin{align*}
        \Pr[Z > \eps] &\ge \Pr[\widehat{\theta} \neq \theta] \cdot \Pr[Z > \eps | \widehat{\theta} \neq \theta] \ge \frac{3}{16}\exp(-64m\eps^2). \tag{using \eqref{eqn:pr-test-wrong}}
    \end{align*}
    The probability in the left-hand side is over all the prevalent randomness $\theta, F, S, r$ in the experiment. But now observe that
    \begin{align*}
        Z > \eps \implies \err_{\mcD_{F}}(\mu_T) - \err_{\mcD_{F}}(h_{\theta, F}) > \eps \implies \err_{\mcD_{F}}(\mu_T) - \inf_{h \in \mcH}\err_{\mcD_{F}}(h) > \eps, \tag{using \eqref{eqn:comparator-hypothesis-error}}
    \end{align*}
    and thus, we obtain that
    \begin{align*}
        \Pr_{\theta, F, S, r}\left[\err_{\mcD_{F}}(\mu_T) - \inf_{h \in \mcH}\err_{\mcD_{F}}(h) > \eps\right] \ge \frac{3}{16}\exp(-64m\eps^2).
    \end{align*}
    In particular, this implies the existence of an $F^\star$ such that
    \begin{align*}
        \Pr_{S, r}\left[\err_{\mcD_{F^\star}}(\mu_T) - \inf_{h \in \mcH}\err_{\mcD_{F^\star}}(h) > \eps\right] \ge \frac{3}{16}\exp(-64m\eps^2).
    \end{align*}
    So, if $m \le \frac{\log(1/8\delta)}{64\eps^2}$, we get that
    \begin{align*}
        \Pr_{S, r}\left[\err_{\mcD_{F^\star}}(\mu_T) - \inf_{h \in \mcH}\err_{\mcD_{F^\star}}(h) > \eps\right] \ge 3\delta/2 > \delta,
    \end{align*}
    as desired. The final choice of $M$ for everything above to go through can be $M = \max(8\log(4)/\eps^2, 8m/\eps)$.

    We now move on to prove part (2) of the theorem.
    
    \paragraph{Necessity of $\Omega(\ell \dnat / \eps^2)$.} Let $\mcX = [d] \times \N$, and $\mcY = [\ell] \times \{-1,1\}$. We think of $(j, +1)$ and $(j, -1)$ as an ``opposite pair''.
    
    Now, for $\theta \in \{-1,1\}^\ell$, define
    \begin{align}
        \label{eqn:def-T-theta}
        A_\theta = \{(j,\theta_j): j \in [\ell]\}.
    \end{align}
    That is, for each $j$, $A_{\theta}$ picks out exactly one out of the pair $(j, +1)$ and $(j, -1)$, based on $\theta_j$.
    
    The hypothesis class $\mcH$ comprises all the functions $h: \mcX \to \mcY$, such that for every ``block'' $b \in [d]$, there exists $\theta^{(b)} \in \{-1,1\}^\ell$ satisfying
    \begin{align*}
        h(\{b\} \times \N) \subseteq A_{\theta^{(b)}}.
    \end{align*}
    Namely, for each block $b$, a hypothesis must make one global sign choice $\theta^{(b)}_j$ to choose from $(j, +1)$ and $(j, -1)$; it cannot use both of these within the block.

    We first argue\footnote{In fact, it can also be checked that $\dds(\mcH)=\ell d$.} that $\dnat(\mcH)=d$. Consider the points $x_1,\dots,x_d = (1,1),(2,1),\dots, (d,1)$ (i.e., one point per block). Then, we can observe there is a hypothesis in the class that realizes every label sequence in 
    \begin{align*}
        \{(1,+1), (2,+1),\dots,(\ell, +1), (1,-1)\}^{\times d}
    \end{align*}
    on $x_1,\dots,x_d$ (since each of these points belongs to a different block), and hence, $\dnat(\mcH) \ge d$. Now consider any sequence of $d+1$ distinct points in $\mcX$. By the pigeonhole principle, this sequence must include two points $x=(b,t)$ and $x'=(b,t')$ from the same block $b$. Suppose that the sequence is $\ell$-Natarajan shattered, and the list of $(\ell+1)$ distinct labels that witnesses the shattering for $x$ is $(j_1, s_1),\dots, (j_{\ell+1}, s_{\ell+1})$. Then, observe that the list for $x'$ cannot contain $(j_i, -s_i)$ for any $i \in [\ell+1]$, since no hypothesis can realize both $(j_i, s_i)$ and $(j_i, -s_i)$ on different elements of $\mcX$ within the same block. In particular, this precludes $(j_1,-s_1),\dots,(j_{\ell+1}, -s_{\ell+1})$, all of which are distinct, from belonging to the list for $x'$. But then, there are only $2\ell - (\ell+1) = \ell-1$ remaining available labels for $x'$, contradicting that the list for $x'$ is of size $\ell+1$. We conclude that $\dnat(\mcH) \le d$.
    
    We will now specify a random experiment similar to the one used for proving the lower bound of $\Omega(\log(1/\delta)/\eps^2)$ above; however, this experiment will be a slightly more involved, ``tensorized'' form of the experiment considered there. Fix $\mcA$ to be any arbitrary learner, possibly using randomness $r$. Similar to before, we will assume without loss of generality that the predictor output by the learner always predicts exactly $\ell$ distinct labels for any $x$.

    To set up the experiment, for any $\theta \in \{-1,1\}^\ell$, define a distribution $p_{\theta}$ over $\mcY$ as
    \begin{align*}
        p_{\theta}((j, s)) = \begin{cases}
            \frac{1+32\eps}{2\ell} & \text{if $s=\theta_{j}$}, \\
            \frac{1-32\eps}{2\ell} & \text{otherwise}.
        \end{cases}
    \end{align*}
    Thus, among $(j,+1)$ and $(j, -1)$, $p_\theta$ prefers $(j, \theta_j)$.
    These probabilities sum up to $1$ over all $j$ and $s$, and since $\eps \le 1/64$, it holds that $\frac{1-32\eps}{2\ell} > 0$. %
    Let $M$ be a suitably large integer to be chosen later. We can now specify the random experiment:
    \begin{enumerate}
        \item[(a)] Draw $d\ell$ hidden bits 
        \begin{align*}
            \Theta = (\Theta_{b, j})_{b \in [d], j \in [\ell]}
        \end{align*}
        uniformly at random from $\{-1,1\}^{d\ell}$.
        \item[(b)] Conditioned on $\Theta$, and denoting $\Theta_b = (\Theta_{b,j})_{j \in [\ell]}$, draw labels $F((b,t))$ independently for every $x=(b,t) \in [d] \times [M]$ from $p_{\Theta_b}$. Thus, all the labels for a fixed block $b$ are i.i.d.\ draws from $p_{\Theta_b}$.
        \item[(c)] Conditioned on $F$, let $\mcD_F$ be the uniform distribution over the set
        \begin{align*}
            \{((b,t), F((b,t))): b \in [d], t \in [M]\}.
        \end{align*}
        \item[(d)] Draw a training sample $S = \{(x_1,y_1),\dots,(x_m,y_m)\} \sim \mcD_F^m$, and run the learner $\mcA$, using internal randomness $r$, on this sample.
    \end{enumerate}

    We will now specify a benchmark hypothesis that we will use to compare the predictor output by the learner. For any realization of $(\Theta, F)$, define the hypothesis $h_{\Theta, F} \in \mcH$ as follows: for any $x=(b,t)$ from block $b$, where $t \in [M]$, it uses precisely $\Theta_b$ as its reference: if $F((b,t)) \in A_{\Theta_b}$, then $h_{\Theta, F}((b,t))=F((b,t))$; otherwise, $h_{\Theta,F}((b,t))$ is some arbitrary fixed element in $A_{\Theta_b}$. For $t > M$, we similarly fix $h_{\Theta, F}$ to be some arbitrary fixed element in $A_{\Theta_b}$.

    We have that 
    \begin{align*}
        \inf_{h \in \mcH}\err_{\mcD_{F}}(h) \le \err_{\mcD_{F}}(h_{\Theta, F}) = \E_{((b,t),F((b,t))) \sim \mcD_{F}}[\Ind[F((b,t)) \notin A_{\Theta_b}]].
    \end{align*}

    Now, let $T=(S, r)$ denote the transcript of the given learner $\mcA$, which includes the training sample $S$ and its internal randomness $r$. The transcript $T$ completely determines the list predictor $\mu_T = \mcA(S,r)$ output by the learner. We will now analyze the distribution of the transcript $T$ separately when the underlying hidden bits are fixed at $\Theta$ and $\Theta^{\oplus (b,j)}$, where $\Theta^{\oplus(b,j)}$ flips the bit at position $(b,j)$ in $\Theta$. 
    
    Namely, let $P^\Theta(T)$ be the distribution of the transcript $T$ conditioned on the hidden bits being $\Theta$. Since $r$ is independent of the training sample, we have that
    \begin{align*}
        \KL(P^\Theta(T) || P^{\Theta^{\oplus (b,j)}}(T)) = \KL(P^\Theta(S) || P^{\Theta^{\oplus (b,j)}}(S)).
    \end{align*}
    Now, condition on any fixed sequence $(x_1,\dots,x_m)$ of unlabeled samples, and consider the conditional distribution of the labels $y_1,\dots,y_m$. The label sequence is determined by the draws $F(x)$ for all the \textit{distinct} values of $x=(b', t')$ in the unlabeled sequence; furthermore, conditioned on the unlabeled sequence, these are independent draws from $p_{\Theta_{b'}}$. Moreover, since $\Theta$ and $\Theta^{\oplus(b,j)}$ differ only at position $(b,j)$, the label distributions in the two cases are identical, except for when $x=(b,t)$. So, let us denote by $m'(b)$ the number of distinct values of $t$ for which $x=(b,t)$ exists in the unlabeled sequence. Note that $\E[m'(b)] \le m/d$, since the probability of observing an $x$ from block $b$ in the sample is exactly $1/d$. Observing further that the unlabeled samples have the same distribution under both $\Theta$ and $\Theta^{\oplus (b,j)}$, we get that
    \begin{align*}
        \KL(P^\Theta(S) || P^{\Theta^{\oplus (b,j)}}(S)) &= \E_{x_1,\dots,x_m}\left[\KL(P^\Theta(y_1,\dots,y_m|x_1,\dots,x_m) || P^{\Theta^{\oplus (b,j)}}(y_1,\dots,y_m|x_1,\dots,x_m))\right] \\ 
        &= \E_{x_1,\dots,x_m}\left[ m'(b) \cdot \KL(p_{\Theta_b} || p_{\Theta^{\oplus (b,j)}_b}) \right] = \E[m'(b)] \cdot \KL(p_{\Theta_b} || p_{\Theta^{\oplus (b,j)}_b}) \\
        &\le \frac{m}{d} \cdot \KL(p_{\Theta_b} || p_{\Theta^{\oplus (b,j)}_b})\\
        &= \frac{m}{d} \cdot \left(\left(\frac{1+32\eps}{2\ell}\right)\log\left(\frac{1+32\eps}{1-32\eps}\right)+\left(\frac{1-32\eps}{2\ell}\right)\log\left(\frac{1-32\eps}{1+32\eps}\right)\right) \\
        &= \frac{32m\eps}{d\ell}\log\left(\frac{1+32\eps}{1-32\eps}\right) \\
        &\le \frac{4096 m\eps^2}{d\ell}. \tag{since $\eps < 1/64,  \log\left(\frac{1+32\eps}{1-32\eps}\right) \le 128 \eps$}
    \end{align*}
    Thus, we have that
    \begin{align}
        \label{eqn:pre-assouad-kl-bound}
        \KL(P^\Theta(T) || P^{\Theta^{\oplus (b,j)}}(T)) \le \frac{4096m\eps^2}{d\ell}.
    \end{align}
    We will now require the following standard form of Assouad's lemma:

    \begin{lemma}[Assouad's Lemma (Theorem 2.12 in \cite{tsybakov2009introduction})]
        Let $\{P_\Theta: \Theta \in \{-1,1\}^N\}$ be a family of distributions. If $\Theta$ is uniform on $\{-1,1\}^N$, then every estimator $\widehat{\Theta}$ satisfies
        \begin{align*}
            \E_{\Theta}[d_H(\widehat{\Theta}, \Theta)] \ge \frac{N}{2}\left( 1-\max_{\Theta, i} \mathrm{TV}(P_\Theta, P_{\Theta^{\oplus i}})\right),
        \end{align*}
        where $d_H$ is the Hamming distance.
    \end{lemma}
    Invoking Assouad's lemma in our calculations above, any estimator $\widehat{\Theta}$ for $\Theta$ satisfies
    \begin{align}
        \E_{\Theta}[d_H(\widehat{\Theta}, \Theta)] &\ge \frac{d\ell}{2}\left( 1-\max_{\Theta, (b,j)} \mathrm{TV}(P^\Theta, P^{\Theta^{\oplus (b,j)}})\right) \ge \frac{d\ell}{4}\exp\left( -\max_{\Theta, (b,j)}\KL(P^\Theta || P^{\Theta^{\oplus (b,j)}})\right). \tag{Bretagnolle–Huber inequality} \\
        &\ge \frac{d\ell}{4}\exp\left(-\frac{4096m\eps^2}{d\ell}\right), \label{eqn:assouad-kl-bound}
    \end{align}
    where in the last inequality we used \eqref{eqn:pre-assouad-kl-bound}. Now, if $m \le \frac{d\ell}{32768\eps^2}$, we get that
    \begin{align*}
        \E_{\Theta}[d_H(\widehat{\Theta}, \Theta)] \ge \frac{d\ell}{4}\exp(-1/8) \ge \frac{7}{32}d\ell.
    \end{align*}
    This implies that
    \begin{align}
        \label{eqn:assouad-high-prob-event}
        \Pr\left[d_H(\widehat{\Theta}, \Theta) \ge \frac{3}{32}d\ell\right] \ge \frac{4}{29}.
    \end{align}
    
    Next, we proceed to lower-bound the excess error of the predictor $\mu_T=\mcA(S,r)$ output by the learner over the error of the benchmark hypothesis $h_{\Theta, F}$, by relating it to a specific estimator $\widehat{\Theta}$ that we will construct ahead.

    For any $x=(b,t)$, consider the random variable
    \begin{align*}
        &Z_x := \Ind[F((b,t)) \notin \mu_T(x)] - \Ind[F((b,t)) \neq h_{\Theta, F}(x)] = \Ind[F((b,t)) \in A_{\Theta_{b}}] - \Ind[F((b,t)) \in \mu_T(x)], \\
        &Z := \frac{1}{dM}\sum_{b=1}^{d}\sum_{t=1}^M Z_x = \err_{\mcD_{F}}(\mu_T) - \err_{\mcD_{F}}(h_{\Theta, F}).
    \end{align*}

    For any block $b \in [d]$, let $O_b$ be the distinct values of $(b,t)$ such that $(b,t)$ is observed in the sample $S$, and let $U_b = (\{b\} \times [M]) \setminus O_b$ be the unobserved values from this block. For $j \in [\ell]$, define:
    \begin{align*}
        C^{(b,j)}_{+1} := \sum_{x \in U_b} \Ind[(j, +1) \in \mu_T(x)], \qquad C^{(b,j)}_{-1} := \sum_{x \in U_b} \Ind[(j, -1) \in \mu_T(x)].
    \end{align*}
    Define the following estimator for $\Theta_{b,j}$:
    \begin{align*}
        \widehat{\Theta}_{b, j} = \argmax_{s \in \{+1, -1\}} C^{(b,j)}_{s}. 
    \end{align*}

    Let us now condition on $\Theta, T$, which completely determines the predictor $\mu_T$, the observed data $O=\cup_b O_b$ the unobserved data $U=\cup_b U_b$ and the estimator $\widehat{\Theta}$. Note that while the labels on $O$ get determined, the labels on any $x=(b,t) \in U$ remain random; in fact, they are mutually independent draws from $p_{\Theta_b}$.

    Since we assumed that $\mu_T$ always predicts exactly $\ell$ distinct labels, we have that
    \begin{align}
        &\sum_{j \in [\ell]}\sum_{x \in U_b} \left( \Ind[(j, -\Theta_{b,j}) \in \mu_T(x)] + \Ind[(j, \Theta_{b,j}) \in \mu_T(x)]\right) = \ell |U_b| \nonumber \\
        \implies & 2\sum_{j \in [\ell]}\sum_{x \in U_b} \Ind[(j, -\Theta_{b,j}) \in \mu_T(x)] = \sum_{j \in [\ell]}\left(\sum_{x \in U_b} \Ind[(j, -\Theta_{b,j}) \in \mu_T(x)] + |U_b|-\sum_{x \in U_b}\Ind[(j, \Theta_{b,j}) \in \mu_T(x)]\right). \label{eqn:massive-parenthesis}
    \end{align}
    Define
    \begin{align*}
        H_b := |\{j \in [\ell]: \widehat{\Theta}_{b,j} \neq \Theta_{b,j}\}|, \qquad H = d_H(\widehat{\Theta}, \Theta) := \sum_{b=1}^d H_b.
    \end{align*}
    Whenever $\widehat{\Theta}_{b,j} \neq \Theta_{b,j}$, by definition of the estimator, it holds that
    \begin{align*}
        \sum_{x \in U_b} \Ind[(j, -\Theta_{b,j}) \in \mu_T(x)] \ge \sum_{x \in U_b} \Ind[(j, \Theta_{b,j}) \in \mu_T(x)].
    \end{align*}
    Otherwise, observe that the expression in the parenthesis \eqref{eqn:massive-parenthesis} is always non-negative. We therefore get
    \begin{align}
        &2\sum_{j \in [\ell]}\sum_{x \in U_b} \Ind[(j, -\Theta_{b,j}) \in \mu_T(x)] \ge |U_b|H_b \implies \sum_{x \in U_b}|\mu_T(x)  \cap A_{-\Theta_b}| \ge \frac{|U_b|H_b}{2} \nonumber \\
        \implies\quad& \sum_{b=1}^d\sum_{x \in U_b}|\mu_T(x)  \cap A_{-\Theta_b}| \ge \sum_{b=1}^d\frac{|U_b|H_b}{2} \ge  \frac{M-m}{2}\sum_{b=1}^d H_b = \frac{H(M-m)}{2}. \label{eqn:main-majority-inequality}
    \end{align}
    where in the last inequality, we used that $|U_b|=M-|O_b| \ge M-m$.
    
    Recall from \eqref{eqn:assouad-high-prob-event} that $H$ is at least $3d\ell/32$ with constant probability. Let us analyze the expected value of $Z_{x}$ for $x=(b,t) \in U_b$, conditioned on 
    $$E := (\Theta, T, H \ge 3d\ell/32).$$
    The labels on the unobserved values of $x$ are still mutually independent draws conditioned on $E$. Namely, for any unobserved $x=(b,t) \in U_b$, we have
    \begin{align*}
        \E[Z_x | E] &= \Pr[F((b,t)) \in A_{\Theta_{b}}] - \Pr[F((b,t)) \in \mu_T(x)] \\
        &= \ell \cdot \frac{1+32\eps}{2\ell} - \left (\frac{1+32\eps}{2\ell}\right)|\mu_{T}(x) \cap A_{\Theta_b}| - \left(\frac{1-32\eps}{2\ell}\right)|\mu_{T}(x) \cap A_{-\Theta_b}| \\
        &= \frac{1+32\eps}{2} - \left(\frac{1+32\eps}{2\ell}\right)(\ell -|\mu_{T}(x) \cap A_{-\Theta_b}|) - \left(\frac{1-32\eps}{2\ell}\right)|\mu_{T}(x) \cap A_{-\Theta_b}| \\
        &= \frac{32\eps}{\ell}|\mu_{T}(x) \cap A_{-\Theta_b}|,
    \end{align*}
    and hence,
    \begin{align*}
        \E[Z | E] &= \frac{1}{dM}\sum_{b=1}^{d}\left(\sum_{x \in U_b} \E[Z_x | E] + \sum_{x \in O_b} \E[Z_x | E]\right) \\
        &\ge \frac{1}{dM}\left(\frac{32\eps}{\ell}\sum_{b=1}^d\sum_{x \in U_b}|\mu_{T}(x) \cap A_{-\Theta_b}|-m\right) \tag{since $Z_x \ge -1, |O| \le m$}\\
        &\ge %
        \frac{16\eps H}{d\ell}\left(1-\frac{m}{M}\right)-\frac{m}{dM}\tag{using \eqref{eqn:main-majority-inequality}} \\
        &\ge \frac{3\eps}{2}\left(1-\frac{m}{M}\right) - \frac{m}{dM} \tag{since $H \ge 3d\ell/32$ on $E$}.
    \end{align*}
    Choosing $M \ge \frac{100m}{\eps}$ ensures that the last quantity above is at least $1.3\eps$ for $\eps \le 1/64$, giving that
    \begin{align}
        \label{eqn:lb-conditional-exp-assouad}
        \E[Z | \Theta, T, H \ge 3d\ell/32] \ge 1.3\eps.
    \end{align}
    But conditioned on $E$, we can also apply Hoeffding's inequality to the mutually independent random variables $Z_x$ for $x \in U$, each of which is bounded in the range $[-1,1]$, to obtain
    \begin{align*}
        \Pr\left[\sum_{x \in U}\left(Z_x - \E[Z_x]\right) \le -0.3dM\eps ~\Big|~ E\right] \le \exp\left(-\frac{d^2M^2\eps^2}{25|U|}\right) \le \exp\left(-\frac{dM\eps^2}{25}\right). \tag{since $|U| \le dM$}
    \end{align*}
    Next, observe that conditioned on $E$, 
    \begin{align*}
        Z \le \eps \implies Z - \E[Z] \le -0.3\eps &\implies \sum_{x \in [d] \times [M]}\left(Z_x - \E[Z_x]\right) \le -0.3dM\eps  \\
        &\implies \sum_{x \in U}\left(Z_x - \E[Z_x]\right)\le -0.3dM\eps,
    \end{align*}
    where in the first implication, we used the lower bound on the conditional expectation from \eqref{eqn:lb-conditional-exp-assouad}, and in the last implication, we used that conditioned on $\Theta$ and $T$, the random variable $Z_x$ for $x \in O$ is determined. Therefore, we get that
    \begin{align*}
        \Pr[Z \le \eps | \Theta, T, H\ge 3d\ell/32] \le  \Pr\left[\sum_{x \in U}\left(Z_x - \E[Z_x]\right) \le -0.3dM\eps ~\Big|~ E\right] \le \exp\left(-\frac{dM\eps^2}{25}\right).
    \end{align*}
    If $M \ge \frac{25\log(4)}{d\eps^2}$, we get
    \begin{align*}
        \Pr[Z \le \eps | \Theta, T, H\ge 3d\ell/32] \le \frac{1}{4}.
    \end{align*}
    Averaging over $\Theta, T$, we get
    \begin{align*}
        \Pr[Z \le \eps | H\ge 3d\ell/32]  \le \frac{1}{4} \implies \Pr[Z > \eps | H\ge 3d\ell/32] \ge \frac{3}{4}.
    \end{align*}
    Finally, this gives
    \begin{align*}
        \Pr[Z > \eps] \ge \Pr[H\ge 3d\ell/32] \cdot \Pr[Z > \eps | H\ge 3d\ell/32] \ge \frac{4}{29}\cdot \frac{3}{4} > \frac{1}{10}.
    \end{align*}
    The probability in the left-hand side is over all the prevalent randomness $\Theta, F, S, r$ in the experiment. But now observe that
    \begin{align*}
        Z > \eps \implies \err_{\mcD_{F}}(\mu_T) - \err_{\mcD_{F}}(h_{\Theta, F}) > \eps \implies \err_{\mcD_{F}}(\mu_T) - \inf_{h \in \mcH}\err_{\mcD_{F}}(h) > \eps,
    \end{align*}
    and thus, we obtain that
    \begin{align*}
        \Pr_{\Theta, F, S, r}\left[\err_{\mcD_{F}}(\mu_T) - \inf_{h \in \mcH}\err_{\mcD_{F}}(h) > \eps\right] > \frac{1}{10}.
    \end{align*}
    In particular, this implies the existence of an $F^\star$ such that
    \begin{align*}
        \Pr_{S, r}\left[\err_{\mcD_{F^\star}}(\mu_T) - \inf_{h \in \mcH}\err_{\mcD_{F^\star}}(h) > \eps\right] > \frac{1}{10}.
    \end{align*}
    The lower bound above goes through as long as $m \le \frac{d\ell}{32768\eps^2}$ and $M = \max(25\log(4)/d\eps^2, 100m/\eps)$. This concludes the proof.
\end{proof}

%% file: discussion.tex
\section{Concluding Thoughts}
\label{sec:concluding-thoughts}

Our proof for upper-bounding $\muh(n)$ by $\dds$ is inherently algebraic. In contrast, classical proofs of this result for the binary case (where the DS dimension becomes equal to the VC dimension) have a distinctly combinatorial flavor \cite{haussler1994predicting,haussler1995sphere}. For example, the elegant proof by \cite{haussler1995sphere} is based on the classical \textit{shifting} operation. In the binary case, the VC dimension and the one-inclusion graph respect nice invariances upon shifting --- shifting does not increase the VC dimension, and does not decrease the number of edges in the one-inclusion graph. When $|\mcY| >2$, neither of these properties hold (see Example 19 in \cite{brukhim2022characterization}). Indeed, these pathologies appear to obstruct most natural inductive proof strategies while working with the DS dimension.

In the discussion section of their paper, \cite{hanneke2026optimal} comment on the implications of combinatorial proofs for DS dimension-related results beyond their mere appeal. Namely, while they used their algebraic characterization to prove an optimal Sauer's lemma for multiclass classes, they speculated that other proofs of their result that have a more combinatorial flavor could lead to improving the sample complexity of multiclass learning. It is thus interesting to us that their algebraic characterization itself turned out to be one way to complete the puzzle.

With regards to future directions, %
it would be very interesting to further explore what could be achieved using algebraic methods in multiclass and list learning.

%% file: appendix.tex
\section{Infinite Label Space Considerations}
\label{sec:infinite-label-space}

We prove the lemma promised in the proof of \Cref{corollary:list-near-optimal-sample-complexity}, which shows that the one-inclusion graph of a class can be list-oriented in a way that ensures every vertex to have outdegree at most the maximum density. When the one-inclusion graph is finite, this result is already shown by Lemma 3.20 in \cite{hanneke2026optimal}. However, for a possibly infinite one-inclusion graph, one needs a compactness argument similar to that given in Appendix B in \cite{brukhim2022characterization} / Appendix A in \cite{charikar2023characterization} / Appendix C.1 in \cite{pmlr-v272-pabbaraju25a}.

\begin{lemma}[Small Outdegree Orientation]
    \label{lemma:compactness}
    For any class $W \subseteq \mcY^n$, there exists an $\ell$-list orientation $\sigma^\ell$ of the edges of $\mcG(W)$ such that
    \begin{align*}
        \max_{v \in \mcG(W)} \outdeg^\ell(v;\sigma^\ell) %
        \le \left\lceil\mu^\ell_{W}(n)\right\rceil.
    \end{align*}
\end{lemma}
\begin{proof}
    Let $\mcG(W)=(V,E)$, where $\mcG(W)$ could possibly be infinite, and let $\mcZ$ be the set of pairs $\mcZ=\{(v,e) \in V \times E:v \in e\}$. For $z=(v,e) \in \mcZ$, define the discrete topology on the set $X_{z}=\{0,1\}$ with $\tau_z = 2^{X_z}=\{\emptyset, \{0\}, \{1\}, \{0,1\}\}$ as the open sets. By Tychonoff's theorem, $\mcK = \prod_{z \in \mcZ}X_{z} = \{0,1\}^{\mcZ}$ is a compact set with respect to the product topology. The base open sets in the product topology are sets of the form $\prod_{z \in \mcZ}S_z$ where each $S_z \subseteq \{0,1\}$, and only a finite number of the $S_z$'s are not equal to $X_z$.
    
    Any $\ell$-list orientation $\sigma^\ell$ of $\mcG(W)$ maps to an element $\kappa \in \mcK$: namely, if $v \in \sigma^\ell(e)$, we set $\kappa_{(v,e)}=1$, and if $v \notin \sigma^\ell(e)$, we set $\kappa_{(v,e)}=0$.

    We now define some ``per-vertex-good'' sets: these are all the orientations for which $v$ has the required $\ell$-outdegree. For every $v \in V$, let $A_v \subseteq \mcK$ be the set 
    \begin{equation}
        \label{eqn:A_v}
        A_v=\left\{\kappa \in \mcK: \left|\left\{e \ni v: \kappa_{(v,e)} = 0\right\}\right| \le \left\lceil\mu^\ell_{W}(n)\right\rceil \right\}.
    \end{equation}
    
    Next, we define some ``per-edge-valid'' sets: these are all the orientations that are valid, in that no edge is oriented towards more than $\ell$ vertices. For every $v \in V$ and $j \in [n]$, let $e_j$ be the edge adjacent to $v$ in the $j^{\text{th}}$ direction. Let $B_{v,j} \subseteq \mcK$ be the set 
    \begin{equation}
        \label{eqn:B_v_j}
        B_{v,j}=\{\kappa \in \mcK: \kappa_{(u,e_j)}=1 \text{ for at most } \ell \text{ vertices } u \in e_j\}.
    \end{equation}
    Now, we claim that every $A_v$ is a closed set in the product topology. To see this, fix any $t$ such that $\left\lceil\mu^\ell_{W}(n)\right\rceil < t \le n$. Next, fix some $t$ edges $E_t=\{e_{i_1},\dots,e_{i_t}\}$ adjacent to $v$. Let $\mcZ_1=\{(v,e): e \in E_t\}$ and $\mcZ_2 = \mcZ \setminus \mcZ_1$. Then,
    \begin{align}
        \label{eqn:barAv}
        \bar{A}_v = \bigcup_{\left\lceil\mu^\ell_{W}(n)\right\rceil < t \le n} \bigcup_{ \substack{E_t=\{e_{i_1},\dots,e_{i_t}\} \\ v \in e_{{i_{p}}} \,\forall p \in [t] }} \underbrace{\prod_{z \in \mcZ_1} \{0\} \times \prod_{z \in \mcZ_2}\{0,1\}}_{\text{base open set}}.
    \end{align}
    Note that every term inside the double-union is a base open set, because every $\mcZ_1$ is finite. Because an arbitrary union of open sets is open, $\bar{A}_v$ is open, and hence $A_v$ is a closed set.

    Similarly, we claim that every $B_{v,j}$ is a closed set. To see this, fix any $t \ge \ell+1$. Fix any $U_t = \{u_1, \dots, u_t\} \subseteq e_j$ of size $t$. Let $\mcZ_1 = \{(u,e_j):u \in U_t\}$ and $\mcZ_2 = \mcZ \setminus \mcZ_1$. Then, the complement of the set $B_{v,j}$ is 
    \begin{align*}
        \bar{B}_{v,j} = \bigcup_{t \ge \ell+1} \bigcup_{U_t=\{u_1,\dots,u_t\}\subseteq e_j}  \underbrace{\prod_{z \in \mcZ_1}\{1\} \times \prod_{z \in \mcZ_2}\{0,1\}}_{\text{base open set}}.
    \end{align*}
    Again, every term inside the double-union is a base open set. Thus, $\bar{B}_{v,j}$ is open, which means that $B_{v,j}$ is closed. This also means that the set $\Sigma_v = A_v \cap \bigcap_{j \in [n]}B_{v,j}$ is closed, by virtue of being a finite intersection of closed sets.

    We will now argue that for every finite $V' \subseteq V$, $\bigcap_{v \in V'}\Sigma_v$ is non-empty. This will establish the finite intersection property for the collection of closed sets $\{\Sigma_v\}_{v \in V}$, which will let us claim that $\bigcap_{v \in V}\Sigma_v$ is non-empty. 
    
    Let $(V', E')$ be the finite subgraph induced by $V'$ on $\mcG(W)$. This corresponds to the one-inclusion graph $\mcG(W')$ of some finite $W' \subseteq W$, where $W'$ can be identified with a subset of $[k]^n$ for some finite $k$ (possibly after relabeling). The proof of Lemma 3.20 in \cite{hanneke2026optimal} shows that for any class $W' \subseteq [k]^n$, there exists an $\ell$-list orientation $\sigma^\ell$ of the edges of $\mcG(W')$ such that
    \begin{align*}
        \max_{v \in \mcG(W')} \outdeg^\ell(v;\sigma^\ell) \le \left\lceil \max_{\mcF \subseteq W'} \dens^\ell(\mcF)\right\rceil \le \left\lceil\mu^\ell_{W'}(n)\right\rceil \le \left\lceil\mu^\ell_{W}(n)\right\rceil,
    \end{align*}
    where the last inequality holds because $W'$ is a finite subset of $W$. We can interpret this $\sigma^\ell$ as defining a $\kappa \in \bigcap_{v \in V'}\Sigma_v$. Namely, for every $v \in V'$, we can find the edges $e$ adjacent to it that are list-oriented towards it, and set $\kappa_{(v,e)}=1$ for these edges. Similarly, for the edges $e$ that do contribute to its $\ell$-outdegree --- there are at most $\left\lceil\mu^\ell_{W}(n)\right\rceil$ many of these --- we set $\kappa_{(v,e)}=0$. Observe that $\kappa$ then satisfies the condition \eqref{eqn:A_v} for $A_v$. Now, we need to check the condition \eqref{eqn:B_v_j} for $B_{v,j}$ --- let $e_j$ be the edge adjacent to $v$ in the direction $j$. For any vertex $u \in V \setminus V'$ that also belongs to $e_j$, we simply set $\kappa_{(u, e_j)}=0$. In this way, $\kappa_{(u, e_j)}=1$ only for the vertices in $V'$. Since $\sigma^\ell$ is a valid $\ell$-list orientation, and by the way we set $\kappa$ to 1 only for the vertices that an edge is oriented toward, we thus satisfy $B_{v,j}$ for all $v$ and $j$. For any $(u,e)$ pair that is left over, we can, e.g., set $\kappa_{(u,e)}=0$. %
    Then, we have shown the existence of a valid $\kappa$ in $\bigcap_{v \in V'}\Sigma_v$, which means in particular that this set is non-empty.

    This further implies, by the finite intersection property, that $\bigcap_{v \in V}\Sigma_v$ is non-empty. Then, let $\kappa^\star \in \bigcap_{v \in V}\Sigma_v$. We can now extract an $\ell$-list orientation of $\mcG(W)$ from $\kappa^\star$. Namely, for any edge $e$, we orient it towards precisely those vertices $v \in e$ for which $\kappa^\star_{(v,e)}=1$. The per-edge condition \eqref{eqn:B_v_j} guarantees that every edge is oriented towards at most $\ell$ vertices. The per-vertex condition also guarantees that no vertex has $\ell$-outdegree larger than $\left\lceil\mu^\ell_{W}(n)\right\rceil$ in this orientation. Finally, if any edge remains unoriented, %
    we can orient it arbitrarily to any $ \le \ell$ vertices it is adjacent to; this does not increase the $\ell$-outdegree of any vertex, and we have obtained the list orientation we desire.
\end{proof}

\section{Multiclass Sample Compression}
\label{sec:multiclass-sample-compression}

We restate \Cref{thm:sample-compression} and sketch the necessary proof details. We refer the reader to \cite{pabbaraju2024multiclass} and \cite{alon2022theory} for a background on multiclass sample compression schemes and concepts related to partial hypothesis classes like disambiguations.

\SampleCompression*
\begin{proof} We derive the upper and lower bounds separately.
    \paragraph{Upper Bound.}
    The upper bound is derived from the following standard minimax argument: Let $\mcH$ be a class having DS dimension $d_{\mathrm{DS}}$, and consider any sample $S=\{(x_1,y_1),\dots,(x_m,y_m)\}$ that is realizable by $\mcH$. Fix any distribution $\mcP$ over $S$. Since $S$ is realizable by $\mcH$, the distribution $\mcP$ is also realizable by $\mcH$. Now consider the deterministic one-inclusion graph predictor $\mcA$ given in Algorithm 1 in \cite{brukhim2022characterization}. By a leave-one-out + exchangeability argument (e.g., see the proof of Proposition 15 in \cite{brukhim2022characterization}), together with the outdegree bounds given in \eqref{eqn:outdegree-bound}, \eqref{eqn:ell-outdegree-2}, it holds that
    \begin{align*}
        \Pr_{T \sim \mcP^{3d_{\mathrm{DS}}}, (x,y) \sim \mcP}[y \neq \mcA(T)(x)] \le \frac{d_{\mathrm{DS}}}{3d_{\mathrm{DS}}+1} \le 1/3.
    \end{align*}

    By the minimax theorem \cite{v1928theorie}, there exists a distribution $\mcQ$ over sequences $T \in S^{3d_{\mathrm{DS}}}$, such that for every fixed $(x_i,y_i) \in S$,
    \begin{align*}
        \Pr_{T \sim \mcQ}[y_i \neq \mcA(T)(x_i)] \le 1/3.
    \end{align*}
    Consider sampling $j$ sequences $T_1,\dots,T_j$ i.i.d.\ from $\mcQ$, for $j=O(\log m)$. We have, by a Chernoff bound, that
    \begin{align*}
        \Pr_{T_1,\dots,T_j \sim \mcQ^j}\left[y_i \neq \mathrm{Majority}\left[\mcA(T_1)(x_i),\dots,\mcA(T_j)(x_i)\right] \right] \le 1/2m,
    \end{align*}
    where Majority$[\cdot]$ outputs the majority vote over its arguments. By a union bound over $i \in [m]$ and an application of the probabilistic method, we get that there exist $T_1,\dots,T_j$ such that for every $i \in [m]$ simultaneously, it holds that $y_i = \mathrm{Majority}\left[\mcA(T_1)(x_i),\dots,\mcA(T_j)(x_i)\right]$. We thus set the compressed sample to be $S'=(T_1,\dots,T_j)$, and the reconstructed map to be 
    $$x \mapsto \mathrm{Majority}\left[\mcA(T_1)(x),\dots,\mcA(T_j)(x)\right].$$
    Note that $|S'| \le O(d_{\mathrm{DS}}\log m)$ as required.
    
    \paragraph{Lower Bound.} The lower bound is primarily based on the construction in \cite{pabbaraju2024multiclass}, where the sequence of reasoning is as follows: 
    \begin{enumerate}
        \item[(1)] \cite{balodis2023unambiguous} constructed a family of graphs $G$ that near-optimally refute the Alon-Saks-Seymour conjecture, meaning that they have a large separation between chromatic number $\chi(G)$ and biclique partition number $\mathrm{bp}(G)$. In particular, $\chi(G) \ge \mathrm{bp}(G)^{(\log(\mathrm{bp}(G)))^{1-o(1)}}$.
        \item[(2)] Starting from this family of graphs, \cite{alon2022theory} constructed a binary \textit{partial} hypothesis class that is hard to disambiguate. Because a sample compression scheme implies disambiguation, this implies a sample compression lower bound. In particular, \cite{alon2022theory} constructed a binary partial hypothesis class having VC dimension 1 for which any sample compression scheme must have size $\Omega((\log m)^{1-o(1)})$ for samples of size $m$.
        \item[(3)] \cite{pabbaraju2024multiclass} showed that by using a unique label per partial hypothesis, one obtains a multiclass hypothesis class that disambiguates the partial class, has DS dimension 1, and inherits the sample compression lower bound of $\Omega((\log m)^{1-o(1)})$.
    \end{enumerate}
    Recently, \cite{pabbaraju2026optimal} tightened the Alon-Saks-Seymour bound in (1) above, getting rid of the $1-o(1)$ exponent. Basing the remaining steps on this improved family of graphs, one obtains a multiclass hypothesis class having DS dimension 1 for which a sample compression lower bound of $\Omega(\log m)$ holds. In order to derive a lower bound for any arbitrary value $d$ of DS dimension, one needs to instead construct a ``direct product'' of the DS-1 class.
    
    Indeed, this precise direct product construction was established in the very recent work of \cite{alon2026uniform}. Namely, the first two steps in the lower bound part of the proof of Theorem 2.4 in \cite{alon2026uniform} result in a partial binary hypothesis class $\mcH_{d,n}$ on $dn$ points that has VC dimension equal to $d$, for which every exact disambiguation necessarily has size $n^{\Omega(d\log n)}$. So, we can consider the binary partial class which is the union of $\mcH_{d,n}$ over disjoint domains for increasing $n$, where every hypothesis labels $\star$ entirely outside its own domain. This class also has VC dimension $d$. Thereafter, following the multiclass disambiguation strategy in \cite{pabbaraju2024multiclass}, we can replace the $\star$ label for hypothesis in the class by a unique label. This results in a (total) multiclass hypothesis class $\mcH$ that has DS dimension $d_{\mathrm{DS}}=d$. %

    Any sample compression scheme for $\mcH$ having size $k(m)$ for samples of size $m$ (i.e., the compressor always outputs a subsample of size $\le k(m)$, possibly along with a supporting bit-string of size $\le k(m)$) is also a sample compression scheme for each of the underlying partial hypothesis classes $\mcH_{d,n}$. In particular, the scheme correctly compresses and reconstructs the supports (i.e., the points where hypotheses are non-$\star$) of every hypothesis in $\mcH_{d,n}$, which, recall, has size at most $nd$. So, set $m=nd$: then, if we collect the reconstructions over the compressed sets (plus side-information bit strings) obtained by compressing the supports of all the partial hypotheses in $\mcH_{d,n}$ (possibly padded by repetitions to have size exactly $m$), we obtain a disambiguation of $\mcH_{d,n}$. The number of possible compressions is at most
    \begin{align*}
        \underbrace{m^{O(k(m))}}_{\text{subsample of size $\le k(m)$}} \cdot \underbrace{2^{O(k(m))}}_{\text{bit-string of size $\le k(m)$}} \cdot \underbrace{2^{O(k(m))}}_{\text{binary labeling}} = m^{O(k(m))},
    \end{align*}
    and thus, the compression scheme implies a disambiguation of $\mcH_{d,n}$ of this size. But recall that any disambiguation of $\mcH_{d,n}$ must have size at least $n^{\Omega(d\log n)}$. Thus, it must be the case that
    \begin{align*}
        m^{O(k(m))} \ge n^{\Omega(d\log n)} \implies 
        k(m) \ge \Omega\left(\frac{d \log^2(m/d)}{\log(m)}\right) = \Omega(d \log m)
    \end{align*}
    for every sufficiently large $m \gg d$, and in particular, for every fixed $d$ as $m \to \infty$.

\end{proof}